\documentclass{article} 

\usepackage{geometry}
\usepackage[utf8]{inputenc} 
\usepackage[T1]{fontenc}   
\usepackage{fullpage}

\date{}

\usepackage{color}
\usepackage[utf8]{inputenc} 
\usepackage[T1]{fontenc}    
\usepackage{hyperref}       
\usepackage{url}            
\usepackage{amsfonts}       
\usepackage{comment}
\usepackage{enumerate}
\usepackage{wrapfig}
\usepackage{subcaption}
\usepackage{pifont}
\usepackage{varwidth}
\usepackage{colortbl}
\usepackage{nccmath}

\usepackage{multirow}
\usepackage{algorithm, algorithmic}

\usepackage[utf8]{inputenc} 
\usepackage[T1]{fontenc}
\usepackage{url}
\usepackage{ifthen}
\usepackage{cite}
\usepackage{amsmath}
\usepackage{amssymb}

\usepackage{enumitem}
\usepackage{mdwmath}
\usepackage{blindtext}
\usepackage{eqparbox}
\usepackage{stfloats}

\newcommand\mydots{\hbox to 0.7em{.\hss.\hss.}}

\usepackage[english]{babel}
\usepackage{lipsum}
\usepackage{xcolor}
\usepackage{tikz}

\usetikzlibrary{matrix,shapes,arrows,positioning,chains, calc}

\usepackage{graphicx}
\newtheorem{theorem}{{Theorem}}
\newtheorem{lemma}[theorem]{{Lemma}}

\newtheorem{definition}{{Definition}}

\DeclareMathAlphabet{\mathbfsl}{OT1}{ppl}{b}{it} 

\newcommand{\bG}{\mathbfsl{G}}

\newcommand{\bK}{\mathbfsl{K}}

\newcommand{\bW}{\mathbfsl{W}} 
\newcommand{\bX}{\mathbfsl{X}}

\newcommand{\bh}{\mathbfsl{h}}

\newcommand{\bj}{\mathbfsl{j}} 
 
\newcommand{\bu}{\mathbfsl{u}} 

\newcommand{\bw}{\mathbfsl{w}} 
\newcommand{\bx}{\mathbfsl{x}}
\newcommand{\by}{\mathbfsl{y}} 

\newcommand{\bk}{\mathbfsl{k}}

\DeclareMathOperator*{\argmin}{argmin}

\newcommand{\beq}[1]{\begin{equation}\label{#1}}
	\newcommand{\ee}{\end{equation}} 

\renewcommand{\leq}{\leqslant}

\newcommand{\Cref}[1]{Co\-ro\-lla\-ry\,\ref{#1}}

\newcommand{\zero}{{\mathbf 0}}

\usepackage{color}

\title{\bf{Coded Machine Unlearning}}
\author{Nasser Aldaghri \\ \small University of Michigan\\ \footnotesize \texttt{\href{mailto:aldaghri@umich.edu}{aldaghri@umich.edu}} \and
Hessam Mahdavifar \\ \small University of Michigan\\ \footnotesize \texttt{\href{mailto:hessam@umich.edu}{hessam@umich.edu}} \and
Ahmad Beirami \\ \small Facebook AI \\\footnotesize \texttt{\href{mailto:beirami@fb.com}{beirami@fb.com}}}
\begin{document}
\maketitle

\begin{abstract}
There are applications that may require removing the trace of a sample from the system, e.g., a user requests their data to be deleted, or corrupted data is discovered.
Simply removing a sample from storage units does not necessarily remove its entire trace since downstream machine learning models may store some information about the samples used to train them. A sample can be perfectly unlearned if we retrain all models that used it from scratch with that sample removed from their training dataset. When multiple such unlearning requests are expected to be served, unlearning by retraining becomes prohibitively expensive. Ensemble learning enables the training data to be split into smaller disjoint shards that are assigned to non-communicating weak learners. Each shard is used to produce a weak model. These models are then aggregated to produce the final central model. This setup introduces an inherent trade-off between performance and unlearning cost, as reducing the shard size reduces the unlearning cost but may cause degradation in performance. In this paper, we propose a coded learning protocol where we utilize linear encoders to encode the training data into shards prior to the learning phase. We also present the corresponding unlearning protocol and show that it satisfies the perfect unlearning criterion. Our experimental results show that the proposed coded machine unlearning provides a better performance versus unlearning cost trade-off compared to the uncoded baseline.
\end{abstract}

\section{Introduction}\label{intro}
Given the abundance of data, machine learning (ML) has become ubiquitous in the past decade \cite{he2016deep,goodfellow2014generative,vaswani2017attention,radford2019language}. Once an ML model is trained, some samples in the training dataset might be required to be \emph{unlearned} due to various reasons, e.g., to satisfy users' requests of data removal, or due to discovery of corrupt low-quality samples or adversarially modified samples that are specifically created to adversely affect the performance of the ML model.
As ML models may be arbitrarily complex, and may be trained on large datasets, it is important to devise unlearning methods that are efficient and can work with arbitrary models.

\subsection{Our contribution}
In this paper, we explore  perfect machine unlearning for regression problems. We consider an ensemble learning setup similar to the one presented in \cite{bourtoule2019machine} where the dataset is sharded at the master node and assigned to non-communicating weak learners to be trained independently from each other and their models are then aggregated at the master node. In this setup, the cost of unlearning is the time required to retrain the affected weak learners trained on the desired samples, which is directly related to the size of the shards as smaller shards incur less unlearning cost; however, this may be at the cost of degraded performance. Figure \ref{fig:tradeoff} shows a sketch of the performance versus unlearning cost trade-off for the uncoded machine unlearning protocol presented in \cite{bourtoule2019machine}. We aim to design a protocol that operates within the desirable region shown in the figure. In this figure, the original learning algorithm, where a single learner is trained on the entire uncoded training dataset, sets the lower bound of the achievable mean squared error (MSE), which may be at an unreasonable unlearning cost for various applications. We present a new framework for encoding the dataset prior to training that can potentially outperform uncoded machine unlearning in terms of performance versus unlearning cost trade-off as shown in Figure \ref{fig:tradeoff}. We show that the proposed protocol can provide significant improvements in performance when compared to uncoded machine unlearning and discuss certain intuitions behind the proposed protocol's success.

\begin{figure}[t]
  \begin{center}
    \includegraphics[width=0.45\columnwidth]{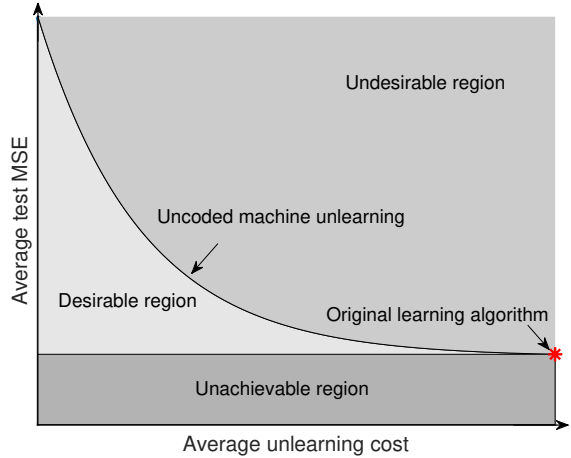}
    \caption{Sketch of performance vs unlearning cost trade-off of the uncoded machine unlearning proposed in \cite{bourtoule2019machine}.}
    \label{fig:tradeoff}
  \end{center}
\end{figure}

More specifically, we consider a regression problem and present a coded learning protocol that utilizes a random linear coding scheme to combine the training samples randomly into a smaller number of samples which are then used to train weak learners with the goal of enabling efficient unlearning. Then, we present an efficient unlearning protocol that utilizes the aforementioned coding scheme to remove the unlearned samples from the training dataset as well as to update the model to completely remove such samples from the model. This is done while maintaining a better performance compared to the uncoded machine unlearning. One of the inspirations of utilizing random codes in this work is the success they have shown in different information processing scenarios such as random codes to achieve channel capacity~\cite{shannon1948mathematical}, random projections in learning kernels~\cite{rahimi2007random,sinha2016learning}, and random projection in compressed sensing~\cite{donoho2006compressed,candes2006robust}. Random projections \cite{rahimi2007random} enable efficient learning of large-scale kernels that are capable of modeling nonlinear relationships. We take advantage of random projections to model nonlinear relationships and propose the use of random linear codes to enable efficient unlearning. Finally, we show the success of the proposed protocol by showing experimental results of the performance against the unlearning cost on a few realistic datasets as well as synthetic datasets.

\subsection{Related work}
The problem of efficiently removing information about a training sample from a trained ML model, referred to as \textit{machine unlearning} \cite{cao2015towards}, has been recently introduced in the literature. In this scenario, a trustworthy party aims to train an ML model on a training dataset of raw data with the guarantee that unlearning requests are satisfied by removing the sample from the training dataset as well as removing any trace of them in the trained ML model. One straightforward approach to perfectly satisfy this requirement is to retrain the model from scratch after removing the samples that need to be unlearned from the training dataset. However, as large training datasets are increasingly available and used in these models, retraining after receiving each unlearning request becomes prohibitively expensive. 

Several works have been proposed in the literature to provide efficient unlearning solutions. Perfect unlearning ensures guaranteed removal of data from learning models. For instance, \cite{cao2015towards} uses statistical query learning to speed up the unlearning process. A more general framework for perfect unlearning considers an ensemble learning setup where a master node shards the training dataset and assigns the shards to non-communicating weak learners that are trained independently from each other, and then aggregates their models using a certain aggregation function \cite{bourtoule2019machine}. 

A different approach to the unlearning problem is known as statistical unlearning, where the data removal protocols offer statistical guarantees of removal, similar to statistical methods for estimating leave-one-out cross validation~\cite{beirami2017optimal,koh2017understanding,giordano2019swiss, rad2020scalable}. For instance, the work in \cite{guo2019certified} presents a statistical formulation of data deletion from machine learning similar to differential privacy and describes a method to achieve deletion in linear and logistic regression scenarios. A formulation of data deletion problems using cryptographic notations is presented in \cite{garg2020formalizing} with a brief discussion on deletion in ML models. Other works on statistical unlearning are also presented in \cite{ginart2019making,neel2020descent,golatkar2020eternal,nguyen2020variational}. Statistical unlearning is typically suitable for convex models which are well-behaving; however, they often provide no guarantees for non-convex models.

\sloppy Another closely related line of work is concerned with individual samples' privacy; this is relevant in cases where the samples contain highly sensitive information and they need to be kept secret even from the ML model. Examples of privacy-preserving ML include works based on differential privacy such as \cite{abadi2016deep,dwork2010boosting,chaudhuri2009privacy} and privacy-preserving learning such as \cite{shokri2015privacy,mohassel2017secureml,so2019codedprivateml,soleymani2020privacy,azam2020towards,li2020npmml}. A major distinction between this line of work and machine unlearning is that samples do not need to be kept private in machine unlearning, but the requests of unlearning need to be honored.

Different methods proposed in the ML literature can be useful in reducing the training cost of ML models. One method used to reduce the number of the randomly projected features of samples is known as data-dependent random projections where feature projections are sampled from some data-dependent distributions \cite{sinha2016learning,shahrampour2017data,agrawal2019data}. Another method is known as data distillation where the dataset is compressed using a distillation algorithm  \cite{radosavovic2018data,wang2018dataset}. However, since these methods are date-dependent, they inherently leak information about samples and their algorithms. Hence, random projections algorithms or distillation algorithms need to be updated accordingly to mitigate any leakage of information about the samples to ensure perfect unlearning after retraining the model. Consequently, this process incurs an additional overhead in the unlearning cost that should not be neglected.

\subsection{Organization}
The rest of the paper is organized as follows. In Section \ref{problem_protocol} the problem setup and the proposed protocol are presented including the proposed protocols for coded learning and unlearning. In Section \ref{experiments} several experiments are conducted to evaluate the proposed protocol using certain datasets along with a discussion on the intuition behind the protocol's success. Finally, the paper is concluded along with a discussion on possible future work directions in Section \ref{conclusion}.

\section{Problem Setup and Proposed Protocol}\label{problem_protocol}

In this section, the setup for the problem of machine learning and unlearning is described along with the proposed protocol for regression models. First, a description of a regression learning model along with its metrics is discussed. Then, the proposed protocol for coded machine learning using data encoding of the training dataset prior to training the learning model is presented and a specific encoder for regression models is introduced. Moreover, the corresponding protocol for perfect coded unlearning is introduced, and its success is shown.
 
\subsection{Problem setup}
Consider a setup where the training dataset is a matrix denoted as $[\bX,\by]$ whose rows are the independent and identically distributed (i.i.d.) samples $\bx_{i}$ along with their response $y_{i}$ for $i=1,2,\mydots,n$, where ${\bx_{i}} \in \mathcal{X}$ and $y_{i} \in \mathcal{Y}$. We denote $n$ as the number of samples, $d$ as the number of features. Columns of $\bX$ are referred to as the features while $\by$ is referred to as the response variable, whose elements are of the form
\begin{align}
    y_i=f(\bx_i)+\epsilon.
\end{align}
The training dataset is used to train a learning model to produce a model, i.e., a function $f:\mathcal{X} \rightarrow \mathbb{R}$, that minimizes a loss function. For regression problems, the loss  $\ell$ is a function that measures the goodness of fit of the model $f \in \mathcal{F}$ on the training dataset, typically expressed as
\begin{align}
    \ell(\bX,\by;f)=\frac{1}{n} \sum_{i=1}^n (y_{i} - f(\bx_i))^2 + \Omega (\|f\|_{\mathcal{F}}), \label{loss}
\end{align}
where $\Omega$ is a regularization term. The learning model finds a function $f^*$ that minimizes the loss function as follows
\begin{align}
    f^*=\argmin_{f \in \mathcal{F}} \ell(\bX,\by;f).\label{optimum_f}
\end{align}
The Representer theorem is a powerful theorem for general regression problems. It states that for the regularized loss in \eqref{loss}, when using a a strictly increasing function $\Omega$, and a kernel $k: \mathcal{X} \times \mathcal{X} \rightarrow \mathbb{R}$ with $\mathcal{F}$ as its associated Reproducing Kernel Hilbert Space (RKHS), then the minimizer $f^*$ of the loss function above is expressed in the form \cite{scholkopf2018learning}
\begin{align}
    f^* = \sum_{i=1}^n w_i k(\cdot,\bx_i),
\end{align}
where $w_i \in \mathbb{R}$. This powerful theorem translates any regression problem, even nonlinear problems, as a linear problem in the RKHS. Hence, the problem can be transformed and re-expressed to be as follows
\begin{align}
    \by=\bK \bw + \boldsymbol{\epsilon},
\end{align}
where $\bK$ is an $n \times n$ kernel matrix whose elements $k_{ij}=k(\bx_i,\bx_j)$, $\bw$ is a $n \times 1$ coefficient vector, and $\boldsymbol{\epsilon}$ is a $n \times 1$ noise vector. The $L_2$-regularized learning model for this kernel problem, also known as ridge regression, aims to estimate $\bw$ that minimizes the loss function
\begin{align}
    \ell(\bK,\by;\bw)=\frac{1}{n} \sum_{i=1}^n (y_{i} - \bk_i^T\bw)^2 + \lambda \bw^T \bw,
\end{align}
where $\bk_i$ is the $i$-th row of $\bK$, and $\lambda$ is the regularization parameter. We denote the resulting model trained on $[\bX,\by]$ when initialized with parameters $\bh$ as $f^* = \mathcal{M}^{\bh}(\bX,\by)$.

These kernel methods suffer greatly in regimes where the size of training datasets is large. Specifically, for a dataset with a fixed number of features, computations of the elements in the kernel matrix result in an additional complexity of $O(n^2)$ on top of the optimization method used to solve the problem. One method of resolving this issue is proposed in \cite{rahimi2007random}, which suggests using random projections of the features to a relatively low-dimensional space compared to $n$. This gives a good approximation of the function $f^*$ using random projections to a $D$-dimensional space where $d < D \ll n$, which enables efficient linear regression methods to be used to solve the regression problem. These random projections enable an approximation of the target function $f^*$, denoted as $\hat{f}$, expressed as follows
\begin{align} 
    \hat{f}(\bx)= \sum_{i=1}^{D} \phi(\bx^T \boldsymbol{\theta}_i+b_i) w_i + \epsilon, \label{randproj}
\end{align}
where $\phi$ is an activation function, $\boldsymbol{\theta}_i$ and $b_i$ are chosen randomly from some distributions, $w_i$'s are the coefficients to be estimated, and $D$ is the desired dimension of the projected features \cite{rahimi2007random}. This enables us to apply this transformation of the original feature matrix $\bX$ into another feature matrix $\bX_{p}$ of size $n \times D$, then we have the following
\begin{align}
    \by = \bX_p \bw + \boldsymbol{\epsilon},\label{y_projected1}
\end{align}
where $\bw$ are the coefficients to be estimated of size $D \times 1$.

After the model has been trained, it is used for prediction until a request of unlearning samples arrives. Once unlearning requests arrive, the model stops processing any prediction requests and launches the unlearning protocol. Machine unlearning is formulated as follows: when an unlearning request of sample $[\bx_{u}^T,y_{u}]$ from the training dataset is received, the model must be immediately updated to remove any effect of this sample, i.e., unlearn it, from $\mathcal{M}^{\bh}(\bX,\by)$. The unlearning protocol is denoted as $\mathcal{U}$ and its output is an updated model denoted as $\mathcal{U}(\mathcal{M}^{\bh}(\bX,\by),[\bx_{u},y_{u}])$. In this paper, we require $\mathcal{U}$ to be a perfect unlearning protocol defined as follows \cite{bourtoule2019machine}.

\begin{definition}[perfect machine unlearning]
An unlearning protocol $\mathcal{U}$ on model $\mathcal{M}^{\bh}(\bX,\by)$ is said to be \emph{perfect} if the output of the unlearning protocol removing the sample $[\bx_{u}^T,y_{u}]$, denoted as $\mathcal{U}(\mathcal{M}^{\bh}(\bX,\by),[\bx_{u}^T,y_{u}])$, is a statistical draw from the distribution of the models trained on $[\bX \setminus  \bx_{u},\by \setminus y_{u}]$, denoted as  $\mathcal{M}^{\bh}(\bX \setminus \bx_{u},\by \setminus y_{u})$, where $[\bX\!\setminus \!\bx_{u},\by\!\setminus\!y_{u}]$ denotes the training dataset $[\bX, \by]$ after removing the sample $[\bx_{u}^T,y_{u}]$ from it.
\label{def:perfect-machine-unlearning}
\end{definition}

Perfect unlearning protocols ensure the complete removal of samples from the model but may suffer in terms of their efficiency. Removing the samples from the training dataset and retraining a model from scratch achieves perfect unlearning. However, the major hurdle of this approach is the extended delay time required to unlearn a sample as retraining is the process that mainly causes this delay; hence, it is desirable to design efficient unlearning protocols that can be used for large scale datasets.

\subsection{Proposed protocol}\label{protocol}
The proposed protocol is described in two parts, learning and unlearning. In the learning phase, we present a method for encoding the training dataset prior to training and describe a specific coding scheme for a regression learning model. After the model has been trained, we transition into the unlearning phase, we describe an efficient method to process unlearning requests using the coded training dataset and update the model to perfectly unlearn the desired samples.

\subsubsection{Learning}\label{protocollearning}

\begin{figure}[t]
  \begin{center}
    \includegraphics[width=0.45\columnwidth]{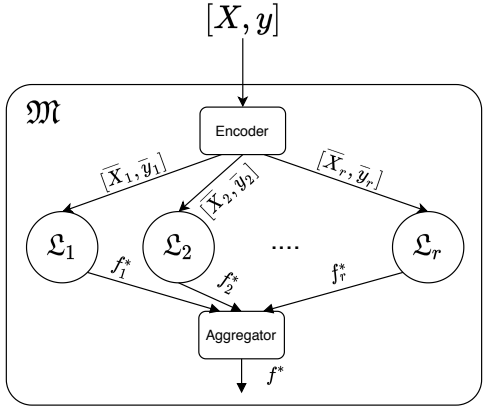}
    \caption{Proposed coded learning setup.}
    \label{fig:model}
  \end{center}
\end{figure}

The proposed protocol introduces the idea of data encoding prior to training the ensemble model as shown in Figure \ref{fig:model}. The learning model $\mathfrak{M}$, also referred to as the master node, is launched to learn a regression model whose training dataset is assumed to have been preprocessed. The model starts by passing the training dataset through an \textit{encoder} to produce a sharded coded training dataset that contains $r$ coded shards. Then, each coded shard $j$ is used to train a weak learner $\mathfrak{L}_j$ to produce a model denoted as $f_j^*$. Once these weak learners are trained, the model $\mathfrak{M}$ is ready for prediction. When sample $\bx$ is passed to the model $\mathfrak{M}$, it is directly passed to each of the weak learners to produce weak predictions $f_j^*(\bx)$ for $j=1,2,\mydots,r$, then the model $\mathfrak{M}$ computes the final prediction $f^{*}(\bx)$ by applying an aggregation function $a: \mathbb{R}^r \rightarrow \mathbb{R}$, such as averaging, a majority vote, etc, as follows
\begin{align}
    f^{*}(\bx)=a(f_1^*(\bx), f_2^*(\bx), \mydots, f_r^*(\bx)).
\end{align}

For linear regression models, or nonlinear regression models coupled with random projections, we know that the generated model $f_j^*$ is the corresponding weights $\bw^*_{j}$. Once all weak learners $\mathfrak{L}_j$'s have been trained, $\mathfrak{M}$ produces a matrix $\bW^*$ whose columns are the estimated coefficients from the weak learners as follows
\begin{align}
    \bW^* = [\bw^*_{1},\bw^*_{2},\mydots,\bw^*_{r}].
\end{align}
Once $\bW^*$ is available, the model $\mathfrak{M}$ computes the aggregate prediction weights by computing the mean of the weight vectors of the weak learners to produce $\bw^*_{\mathrm{agg}}$, which is directly used at the time of prediction, bypassing the weak predictions computations. When the sample $\bx$ is passed as input, the predicted output is
\begin{align}
    f^{*}(\bx)=\bx^T\bw^*_{\mathrm{agg}}.
\end{align}

The encoding of the training dataset is a method to produce a new training dataset with the goal of reducing learning and unlearning costs. Coding can be viewed as a method to incorporate multiple samples from the uncoded training dataset into a single sample of the coded training dataset to enable efficient learning and unlearning. In other words, although we have fewer coded samples for weak learners, each of these coded samples is created from multiple uncoded samples, enabling the model to learn these uncoded samples indirectly. First, let us define an encoder as follows.
\begin{definition}[encoder]
An encoder with rate $\tau$ is defined as a function that transforms the original training dataset with $n$ samples into another dataset with $m$ samples while maintaining the same number of features. The rate of this encoder is
\begin{align}
    \tau = \frac{n}{m}.
\end{align}
\end{definition}
When using shards of equal size, as considered in this work, the rate can also be viewed as the ratio of the number of uncoded shards to the number of coded shards. The rate and design of the encoder are now additional parameters of the model that require tuning when training a model. It is worth noting that the design of the encoder itself for a fixed rate directly affects the unlearning cost of the overall model as well as the performance, as will be clarified later. Hence, it should be carefully considered when designing a model.

\begin{algorithm}[t]
\caption{Learning (Learn)}
\begin{algorithmic}[1]\label{learningalg}
\STATE Input: $[\bX,\by], s, r, \rho$.
\STATE Output: $\bW^*, \{\overline{\bX},\overline{\by}\}, \bG$.
    \STATE \textbf{At master node $\mathfrak{M}$, do}
    \IF {$s \neq 1$}
    \STATE $\{\overline{\bX},\overline{\by}\}, \bG \leftarrow$ $\mathrm{LinearEnc}([\bX,\by], s, r, \rho)$.
    \ELSE 
    \STATE $\{\overline{\bX},\overline{\by}\} = \{[\bX,\by]\}$
    \STATE $\bG = [1]$
    \ENDIF
    \STATE Send $[\overline{\bX}_i,\overline{\by}_i]$ to weak learner $\mathfrak{L}_i$
    \STATE \textbf{At weak learner $\mathfrak{L}_i$, do}
    \STATE $\bw^*_{i} \leftarrow \argmin_{\bw} \ell([\overline{\bX}_i,\overline{\by}_i],\bw)$
    \STATE Send $\bw^*_{i}$ to $\mathfrak{M}$
    \STATE \textbf{At master node $\mathfrak{M}$, do} 
    \STATE $\bW^* = [\bw^*_{1}, \bw^*_{2}, \mydots, \bw^*_{r}]$
    \STATE $\bw^*_{\mathrm{agg}} = \frac{1}{r} \sum_{i=1}^{r} \bw^*_i$
\end{algorithmic}
\end{algorithm}

\begin{algorithm}[t]
\caption{Linear encoder (LinearEnc)}
\begin{algorithmic}[1]\label{linearenc}
\STATE Input: $[\bX,\by]$, $s$, $r$, $\rho$.
\STATE Initialization: $\bG=\zero_{s \times r}$, $\{\overline{\bX},\overline{\by}\}=$ empty.
\STATE Output: $\{\overline{\bX},\overline{\by}\}$, $\bG$.
    \WHILE{$\bG$ is not full column rank}
        \STATE Set $\bG=\zero_{s\times r}$
        \WHILE{$\bG$ has any all-zero rows}
            \STATE $\bG \leftarrow \mathrm{RandMatrix}(s, r,\rho)$
        \ENDWHILE
    \ENDWHILE
    \STATE Split $[\bX,\by]$ into $s$ submatrices $[\bX_i,\by_i]$ of equal size
    \FOR{$j$ in $\mathrm{range}(r)$}
        \STATE $\{\overline{\bX},\overline{\by}\}.\mathrm{append}([(\sum_{i=1}^s g_{ij} \bX_i),(\sum_{i=1}^s g_{ij} \by_i)])$
    \ENDFOR
    \RETURN $\{\overline{\bX},\overline{\by}\}, \bG$
\end{algorithmic}
\end{algorithm}

The proposed coded learning model for linear regression is described in Algorithm \ref{learningalg}. The learning algorithm takes the training dataset $[\bX,\by]$ and code parameters $s, r, \rho$ as inputs and outputs the coded training dataset along with the coding matrix and the trained model. The model $\mathfrak{M}$ utilizes a linear encoder described in Algorithm \ref{linearenc} to encode the training dataset using the provided code parameters. The linear encoder takes the desired code parameters $s, r, \rho$ along with the training dataset $[\bX, \by]$ as inputs and processes it as follows: $[\bX,\by]$ is divided into $s$ disjoint submatrices of equal size, i.e., each has $\overline{n}=\frac{n}{s}$ samples, denoted as $[\bX_i,\by_i]$ for $i=1,2,\mydots,s$. These are encoded using a matrix $\bG$, described next, to produce $[\overline{\bX}_j,\overline{\by}_j]$, for $j=1,2,\mydots,r$. The output of this encoder is $\{\overline{\bX},\overline{\by}\}$ whose elements are the coded shards used to train the corresponding weak learners, i.e., shard $j$ is used to train the $j$-th weak learner to produce the corresponding optimum $\bw^*_{j}$. Note that the code parameters should keep the coded shards in the original regime of the original dataset; for example, if $n\!>d$, then $\overline{n}\!> d$.

Following the success of random codes in information theory~\cite{shannon1948mathematical} and random projections in signal processing~\cite{candes2006robust,donoho2006compressed} and machine learning~\cite{rahimi2007random}, we propose to use a random binary matrix generator (RandMatrix) to generate $\bG$. In this protocol, the matrix $\bG$ is of size $s \times r$ and density $0<\rho \leq 1$. We desire the matrix $\bG$ to be a tall matrix, i.e., $r \leq s$, since our goal is to reduce the number of coded samples used for training. Since $r \leq s$, $\bG$ needs to satisfy two conditions: each row of $\bG$ should have at least one nonzero element, and it should have full rank. The first condition ensures that all the shards are used in training the model, while the second condition ensures that every weak learner has a training dataset that is unique from all other weak learners. Another consequence of using a code with $r \leq s$ is that it lowers the initial learning cost by a factor of $\tau$ compared to uncoded machine unlearning. For example, for some $s$ and $r$ then we only need to train $r$ learners each using $n/s$ coded samples, compared to $s$ learners each with $n/s$ uncoded samples in uncoded machine unlearning.

\subsubsection{Unlearning}\label{protocolUnlearning}

\begin{algorithm}[t]
\caption{Unlearning (Unlearn)}
\begin{algorithmic}[1]\label{unlearningalg}
\STATE Input: $\{\overline{\bX},\overline{\by}\}, [\bX_{\bu}, \by_{\bu}], \bu, \bG, \bW^*$.
\STATE Initialization: $\bj=$ empty.
\STATE Output: $\{\widetilde{\bX},\widetilde{\by}\}, \widetilde{\bW}^*$.
    \STATE \textbf{At master node $\mathfrak{M}$, do}
    \STATE Set $\{\widetilde{\bX},\widetilde{\by}\}=\{\overline{\bX},\overline{\by}\}$
    \STATE Set $\widetilde{\bW}^* = \bW^*$
    \FOR{$i$ in $\mathrm{range(length}(\bu))$}
        \STATE $s' \leftarrow$ index of the uncoded shard containing $[\bx_{i}^T, y_{i}]$
        \STATE $i' \leftarrow$ index of $[\bx_{i}^T, y_{i}]$ within shard $s'$
        \STATE $\bj' \leftarrow$ indices of nonzero elements in row $s'$ of $\bG$
        \FOR{$j'$ in $\bj'$}
            \STATE $\widetilde{\bx}_{j' i'} = \widetilde{\bx}_{j' i'} - g_{s' j'} \bx_{i}$
            \STATE $\widetilde{y}_{j' i'} = \widetilde{y}_{j' i'} - g_{s'j'} y_{i}$
        \ENDFOR
        \STATE $\bj.\mathrm{append}(\bj')$
    \ENDFOR
    \STATE $\bj_{u} \leftarrow \mathrm{unique}(\bj)$
    \STATE Send $[\widetilde{\bX}_j,\widetilde{\by}_j]$ to weak learner $\mathfrak{L}_j$ for all $j \in \bj_{u}$
    \STATE \textbf{At weak learner $\mathfrak{L}_j$, do}
    \STATE $\tilde{\bw}^*_{j} \leftarrow \argmin_{\bw} \ell(\widetilde{\bX}_j,\widetilde{\by}_j;\bw)$
    \STATE Discard the previous model $\bw^*_{j}$
    \STATE Send $\tilde{\bw}^*_{j}$ to $\mathfrak{M}$
    \STATE \textbf{At master node $\mathfrak{M}$, do}
    \STATE Replace column $j$ of $\widetilde{\bW}^*$ with the updated $\tilde{\bw}^*_{j}$ for all $j \in \bj_{u}$
    \STATE Set $\tilde{\bw}^*_{\mathrm{agg}} = \frac{1}{r} \sum_{i=1}^{r} \tilde{\bw}^*_{i}$
\end{algorithmic}
\end{algorithm}

Now that the model has been trained on the coded training dataset, we proceed to describe a protocol to unlearn samples from this model. Our goal is to remove such samples from the coded shards as well as to remove any trace of such samples from the affected weak learners where such samples appear. The unlearning protocol for the aforementioned learning protocol is described in Algorithm \ref{unlearningalg}. The algorithm's inputs are the coded dataset $\{\overline{\bX},\overline{\by}\}$, the samples to be unlearned $[\bX_{\bu}, \by_{\bu}]$, their indices $\bu$ in the uncoded training dataset $[\bX,\by]$, the matrix $\bG$, and the original model estimated coefficients $\bW^*$. The algorithm's outputs are the updated coded dataset $\{\widetilde{\bX},\widetilde{\by}\}$, and the updated estimated coefficients of the model $\widetilde{\bW}^*$. Essentially, the algorithm needs to identify the uncoded shards that include the samples with indices $\bu$ as well as their corresponding coded shards using the matrix $\bG$. The samples first need to be removed from the coded shard by subtracting them from the corresponding coded samples in all coded shards where they appear to eliminate their effect from the coded shards. Once all the coded shards are updated, they are then used to update their corresponding weak learners to unlearn these samples from the weak learner models followed by updating the final aggregate model using the updated weak learners' estimates. The following lemma proves that the algorithm guarantees perfect unlearning.

\begin{lemma}
The unlearning protocol described in Algorithm \ref{unlearningalg} perfectly unlearns the desired samples from the model in the sense of Definition~\ref{def:perfect-machine-unlearning}.
\end{lemma}
\begin{proof}
\sloppy Without loss of generality, we consider a single sample $[\bx_u^T,y_u]$ that is requested to be unlearned from the model, which appears in $[\overline{\bX}_j,\overline{\by}_j]$ that is used to train the $j$-th weak learner whose model is denoted by $\mathcal{M}_j^{\bh}(\overline{\bX}_j,\overline{\by}_j)$. First, the protocol updates this training dataset to be $[\widetilde{\bX}_j,\widetilde{\by}_j]$ by subtracting the sample $[\bx_u^T,y_u]$ from the corresponding coded sample in order to remove it from the dataset $[\overline{\bX}_j,\overline{\by}_j]$. Then, the $j$-th weak learner, whose new training dataset is $[\widetilde{\bX}_j,\widetilde{\by}_j]$, is trained from scratch and the resulting model is denoted as $\mathcal{U}(\mathcal{M}_j^{\bh}(\overline{\bX}_j,\overline{\by}_j),[\bx_{u}^T,y_{u}])$. This model is equivalent to a model $\mathcal{M}_j^{\bh'}(\widetilde{\bX}_j,\widetilde{\by}_j)$, where $\bh'$ is chosen randomly. Using the uniqueness property of the linear and ridge regression solutions, we have the following:
\begin{align}
    \mathcal{U}(\mathcal{M}_j^{\bh}(\overline{\bX}_j,\overline{\by}_j),[\bx_{u}^T,y_{u}]) = \mathcal{M}_j^{\bh''}(\overline{\bX}_j \setminus \bx_u,\overline{\by}_j\setminus y_u),
\end{align}
for some random $\bh''$. Hence, the desired sample is perfectly unlearned from the $j$-th weak learner. The same argument applies to all other affected weak learners after removing the desired samples from their corresponding training datasets. Therefore, as the resulting models from the affected weak learners are updated along with re-calculating the aggregation function, the overall updated model perfectly unlearns the desired samples from the model.
\end{proof}

For large-scale problems, we can speed up the unlearning protocol even more. Using iterative optimization methods one can start the optimization problem for the weak learners on the new training dataset using the solution from their previous model. Specifically, for linear and ridge regression problems the resulting model will always be the same as the one trained from scratch since these iterative methods will converge to a unique global minimizer regardless of the initialization. However, this cannot be used for other complex models such as the over-parameterized multi-layer perceptron (MLP) since the training loss can be zero for these models. Hence, when a sample is removed and the model is initialized from the previous model, it will immediately converge since the training loss is already zero, but this solution was reached in part due to the removed sample. Therefore, this approach in the over-parameterized scenario does not perfectly unlearn the sample.

The last design parameter of the coded learning is the generator matrix $\bG$. One of the properties of the matrix $\bG$ used in the encoder is its density $\rho$, and it can be seen in Algorithm \ref{unlearningalg} that the density of $\bG$ directly affects the unlearning cost. For example, a sample whose corresponding row in $\bG$ is dense requires updating more weak learners than a sample whose corresponding row is sparse. Therefore, the design of such a matrix is directly related to the efficiency of unlearning. If we aim to have the lowest unlearning cost for an encoder with a specific rate, we use the minimum matrix density that satisfies both of the aforementioned conditions for the matrix $\bG$, which is $\rho = \frac{1}{r}$. This corresponds to the case where there is only one nonzero element in each row of the matrix $\mathbf{G}$, i.e., each sample only shows up exactly in one coded shard. 

\emph{Remark:} Since the choice of the encoder in Algorithm \ref{learningalg} is independent of the data, it does not leak any information about the data itself and does not affect the perfect unlearning condition. However, other types of data-dependent encoders may require additional steps to ensure the removal of the unlearned samples from the encoder itself, which may introduce additional overhead. An example of such encoders is one that assigns samples to weak learners based on some properties of the training dataset itself. This leakage of information, even if small, needs to be taken into account when designing perfect unlearning protocols.

\section{Experiments}\label{experiments}
In this section, we present simulation results of some experiments to compare the performance versus the unlearning cost on realistic and synthetic datasets for two protocols: the uncoded machine unlearning protocol described in \cite{bourtoule2019machine} and the proposed coded machine unlearning protocol. The experiments simulate the unlearning of a sample from the training dataset, where the performance is measured in terms of the mean squared error and the unlearning cost is measured in terms of the time required to retrain the affected weak learner.

We utilize the {\fontfamily{qcr}\selectfont sklearn.linear\_model} package \cite{sklearn}, specifically, {\fontfamily{qcr}\selectfont LinearRegression} or {\fontfamily{qcr}\selectfont Ridge} modules, to produce the simulation results for all the experiments. Since the cost of unlearning is related to the size of the shards, we sweep the variable $s$ while fixing the rate for the coded scenario and observe the performance. Each point in the plots shows the average of a number of runs, where each run simulates the experiment on a randomly shuffled dataset that is then split into training and testing datasets according to the specified sizes. During each run, after splitting the dataset into training and testing, Algorithm \ref{learningalg} is run first using $s$ and $r$ for a specific code with rate $\tau = \frac{s}{r}$ and density $\rho = \frac{1}{r}$. Once the model is trained, a random sample from the training dataset is chosen to be unlearned using Algorithm \ref{unlearningalg}. After all the runs are done, the performance is measured as the average mean squared error of the testing dataset, while the unlearning cost is measured as the average time required to retrain the affected weak learners, since removing a sample from the dataset has negligible cost.

For the simulations, datasets are preprocessed as follows, each column of the original feature matrix and the response vector is normalized to be in the range $[0,1]$. If the random projections approximation \cite{rahimi2007random} is used as described in \eqref{randproj}, then the projections are done on the normalized features using a cosine activation function and the following parameters 
\begin{align}
    &\boldsymbol{\theta}_i \sim \mathcal{N}(0,\frac{1}{2d} I_d),\label{theta_vector}\\
    &b_i \sim \mathrm{unif}(-\pi, \pi).\label{bias_element}
\end{align}

\subsection{Results}
\label{sec:results}
We conduct three experiments to evaluate the proposed protocol on realistic datasets. The first dataset is known as the Physicochemical Properties of Protein Tertiary Structure dataset \cite{Dua2019}. The goal is to use the $9$ original features to estimate the root mean square deviation. This dataset includes $45{,}730$ samples, where $42{,}000$ samples are for training, and the rest are for testing. We consider random projections with $D=300$. The results shown in Figure \ref{fig:CASP} show the simulation results for multiple values of $\lambda=10^{-4},10^{-5},10^{-6}$ using a code of rate $\tau=5$. It can be seen that coding provides better performance compared to the uncoded machine unlearning at lower unlearning cost, even when using regularization with different values.

The second dataset is known as the Computer Activity dataset \cite{Compact}. It is concerned with estimating the portion of time that the CPU operates in user mode using different observed performance measures. We consider random projections of the original features to a space with $D=25$. The dataset has $8{,}192$ samples, with $12$ original features, of which $7{,}500$ samples are for training while the rest are for testing. The experiments use a regularization parameter $\lambda = 10^{-3}$ and different code rates $\tau=2, 5$. The results are shown in Figure \ref{fig:CompAct}. In this figure, we observe that coding provides better performance compared to the uncoded machine unlearning at a lower unlearning cost. Additionally, different rates allow for different achievable performance measures as evident in the figure. More on the effect of code rates will be discussed later in the experiments on a large-scale synthetic dataset.

Finally, we experiment on the Combined Cycle Power Plant dataset \cite{Dua2019}. The goal is to estimate the net hourly electrical energy output using different ambient variables around the plant. The dataset has $9{,}568$ samples, with $4$ original features, of which $9{,}000$ samples are for training while the rest are for testing. We consider random projections with $D=20$. The experiments use linear regression with no regularization and different code rates $\tau=2, 5$ and their results are shown in Figure \ref{fig:CCPP}. For this case, there is no region of coding to operate in, and intuitively, we do not expect it to beat the performance of the original learning algorithm with a single uncoded shard. However, although coding does not provide a better trade-off in this case, it does not exhibit a worse trade-off either.

\begin{figure}[t]
\centering
\begin{minipage}{.45\textwidth}
  \centering
  \includegraphics[width=0.95\linewidth]{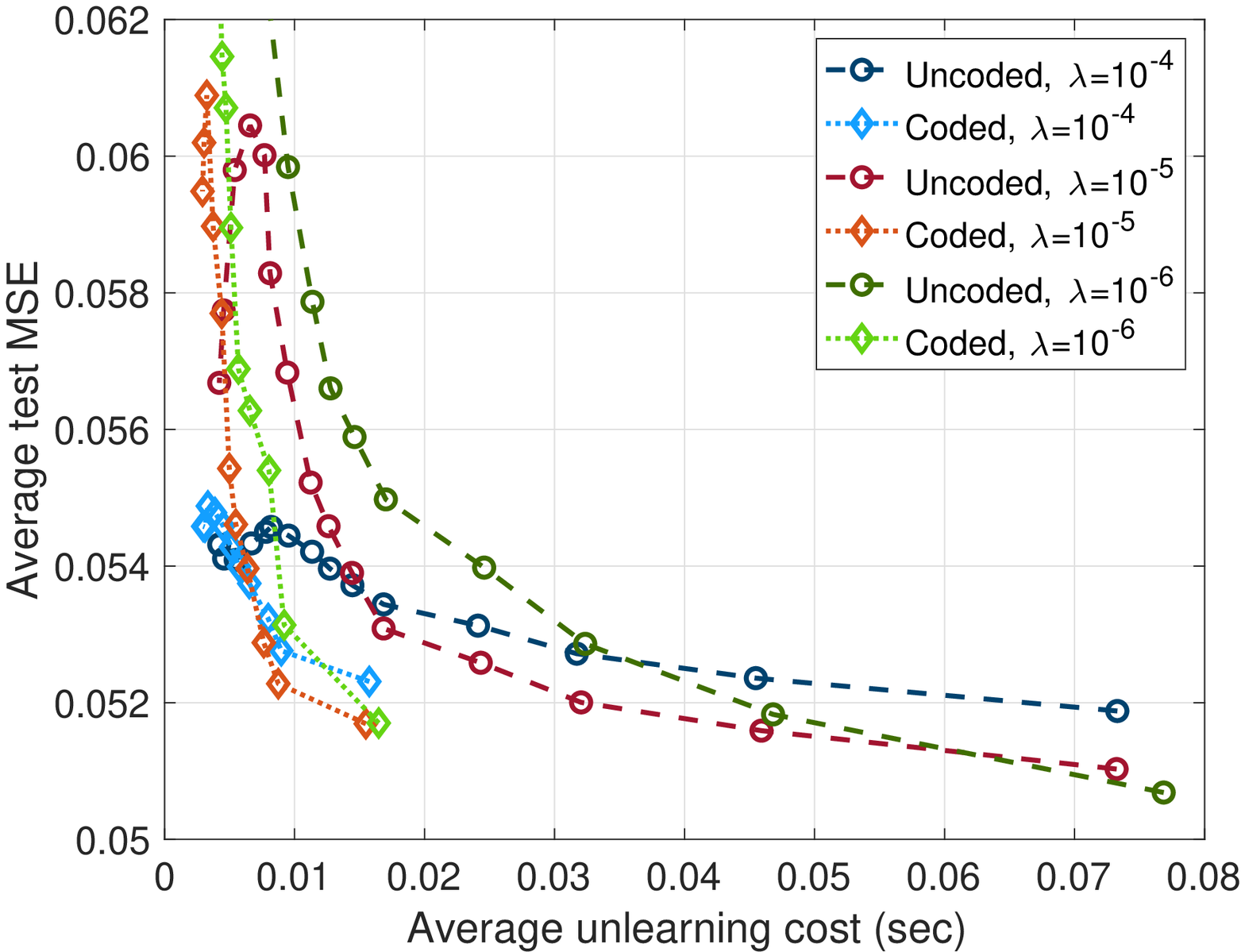}
  \captionof{figure}{Performance vs unlearning cost for different values of $\lambda$ using the Physicochemical Properties of Protein Tertiary Structure dataset \cite{Dua2019}, random projections of features to a $300$-dimensional space, and a code of rate $\tau=5$.}
  \label{fig:CASP}
\end{minipage}
\hfill
\begin{minipage}{.45\textwidth}
    \centering
    \includegraphics[width=0.95\textwidth]{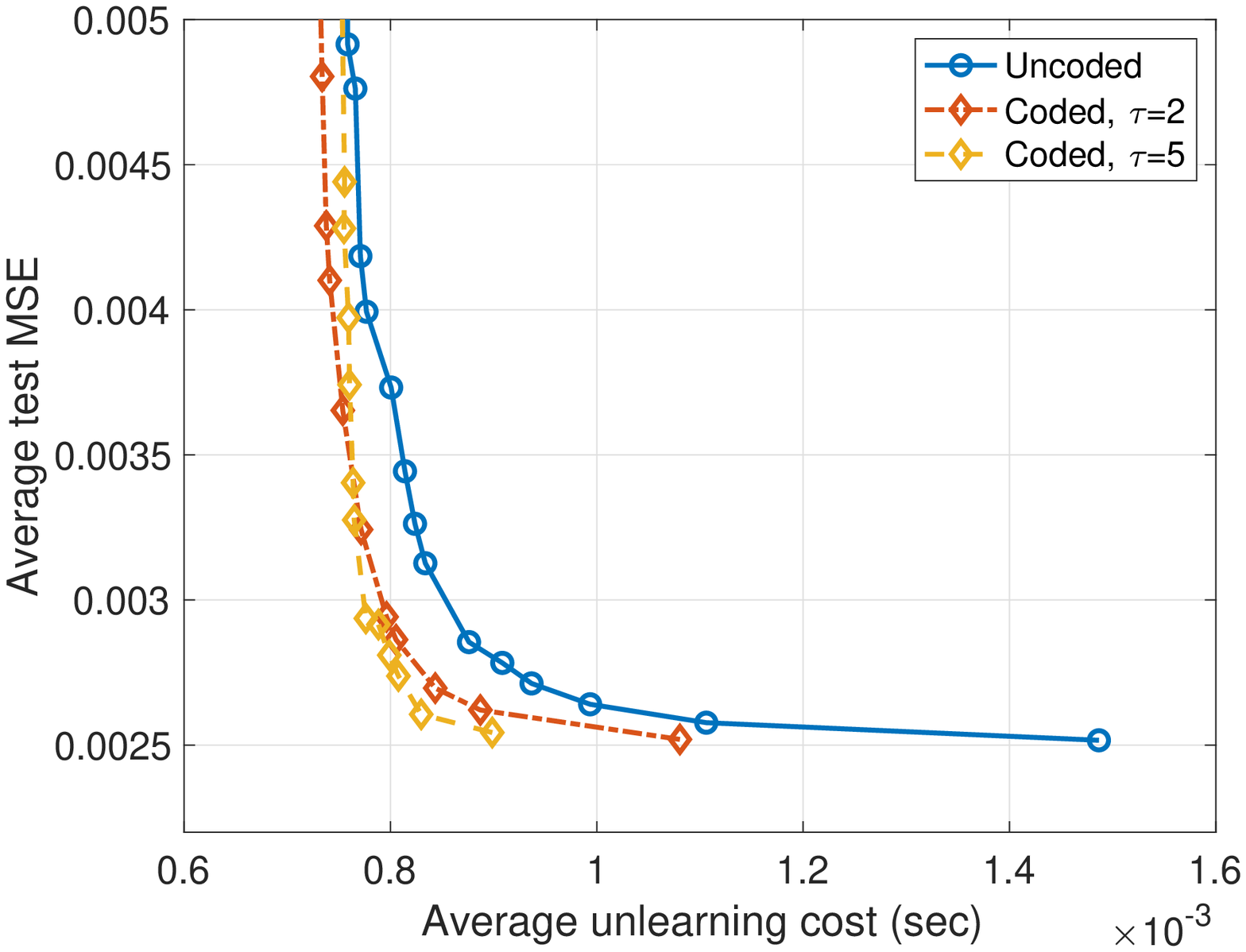}
    \captionof{figure}{Performance vs unlearning cost for different rates $\tau$ using the Computer Activity dataset \cite{Compact}, random projections of features to a $25$-dimensional space, and $\lambda=10^{-3}$.\\}               
    \label{fig:CompAct}
\end{minipage}
\vspace{-0.1in}
\end{figure}

The above experiments show results for datasets with relatively small to moderate size and number of features. It remains to be seen if similar behavior can be observed if the dataset size as well as the number of features become much larger. The following experiment shows simulation results of a synthetic dataset generated as follows: a total of $600{,}000$ samples are generated, each with i.i.d. features of size $d=100$ drawn from lognormal distribution with parameters $\mu=1, \sigma^2=4$, then passed through a random 3-hidden-layers MLP, followed by an output layer with standard normal noise term to generate the desired response variable. The layers contain $50,25,50$ nodes, respectively, with a sigmoid activation function and their weights and biases are i.i.d. drawn from the standard normal distribution. We use $\lambda = 10^{-2}$ and apply random projections on the original features using the parameters described above and $D=2{,}000$. The dataset is split into $500{,}000$ samples for training and $100{,}000$ for testing. The simulation results are shown in Figure \ref{fig:LargeScaleNN}. Note that log-scale is used on the $x$-axis for better showing of the curves for the coded scenarios. It can be observed that as we increase the rate of the code, the unlearning cost decreases while minimum achievable MSE increases. Hence, one can choose the maximum code rate that satisfies a performance close to the original learning algorithm.

\begin{figure}[t]
\centering
\begin{minipage}{.45\textwidth}
  \centering
  \includegraphics[width=0.95\linewidth]{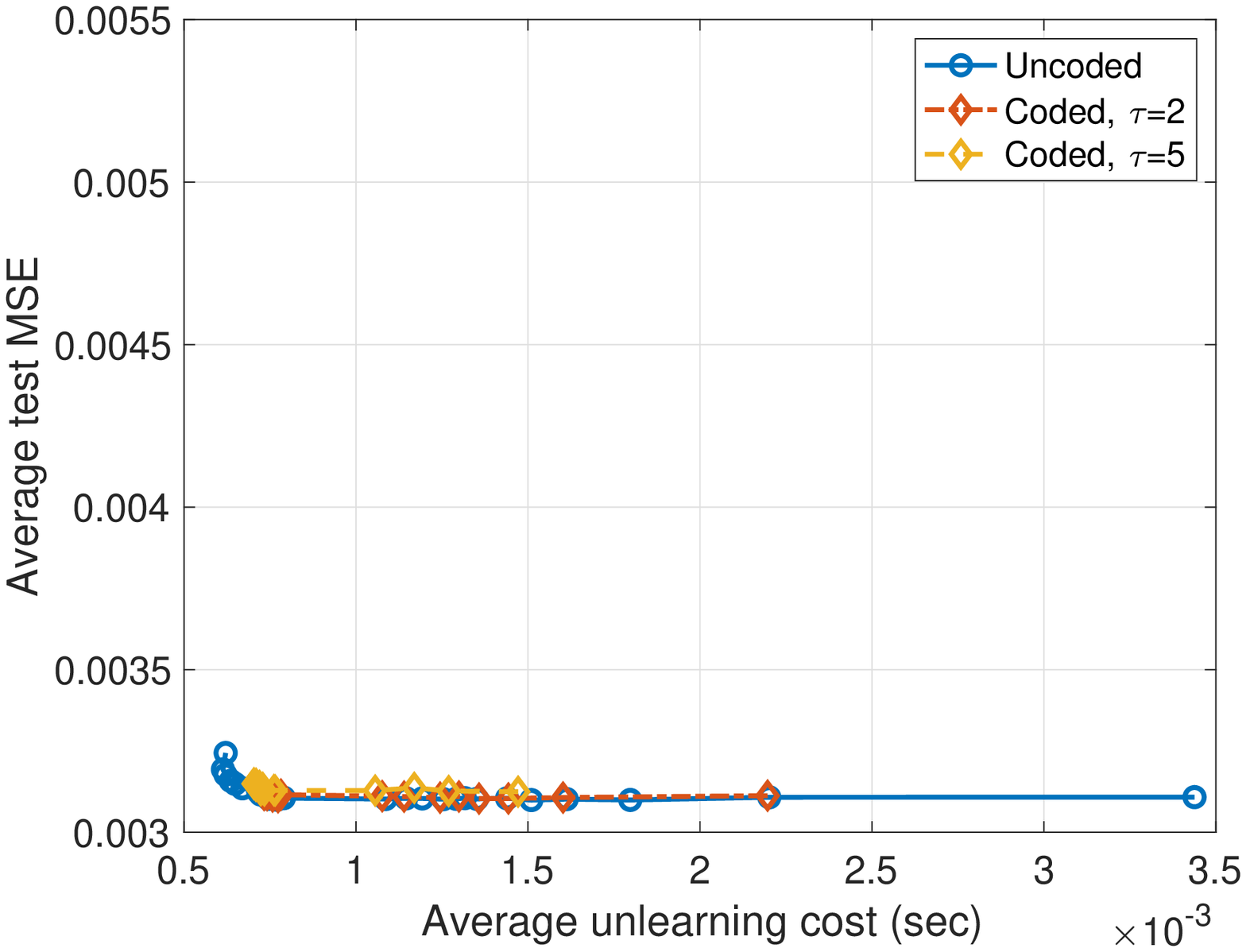}
  \captionof{figure}{Performance vs unlearning cost for different rates $\tau$ using the Combined Cycle Power Plant dataset \cite{Dua2019} and random projections of features to a $20$-dimensional space.\\}
  \label{fig:CCPP}
\end{minipage}
\hfill
\begin{minipage}{.45\textwidth}
    \centering
    \includegraphics[width=0.95\textwidth]{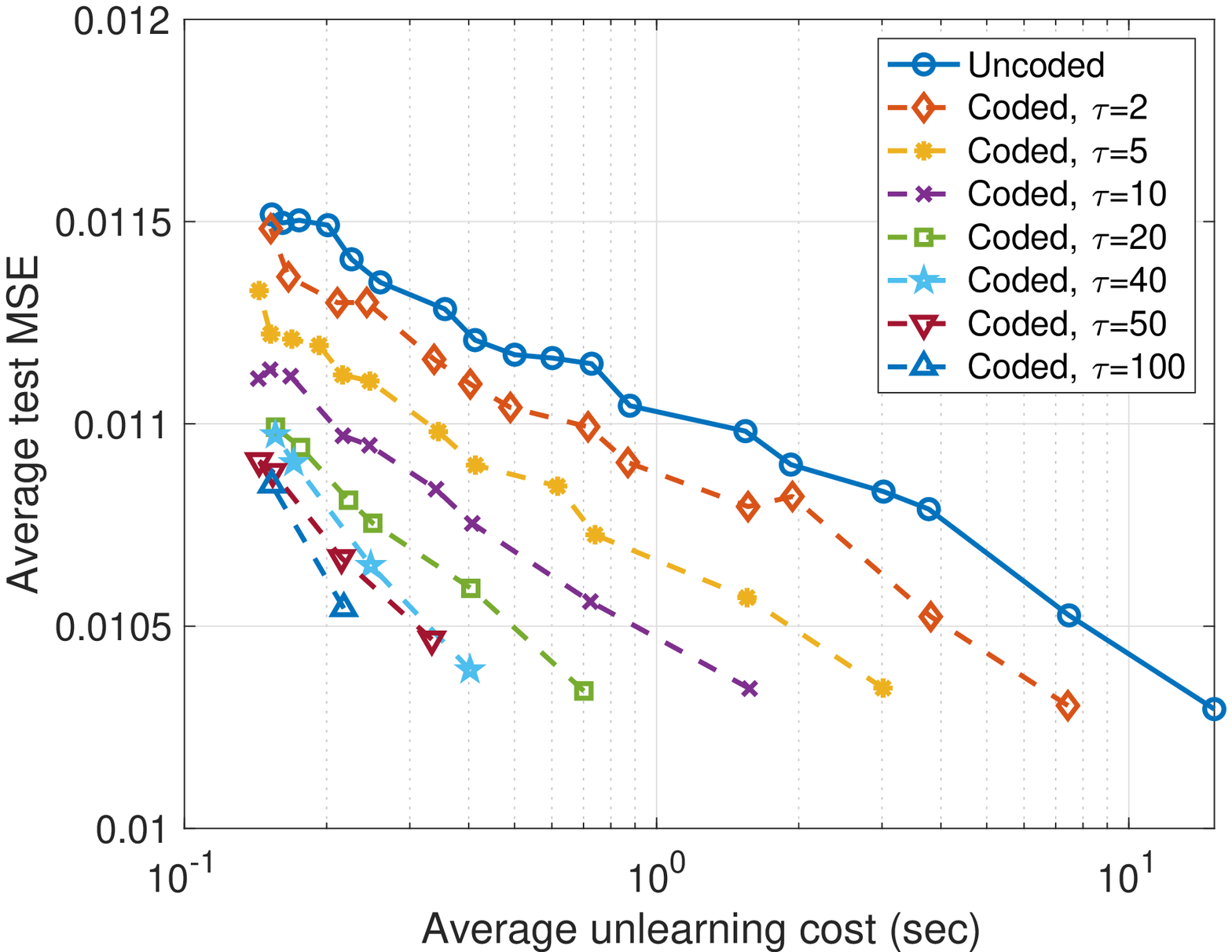}
    \captionof{figure}{Performance vs unlearning cost for synthetic data generated from an MLP with $\mathrm{lognormal}(1,4)$ features using random projections to a $2{,}000$-dimensional space, $\lambda=10^{-2}$, and codes of different rates $\tau$.}
    \label{fig:LargeScaleNN}
\end{minipage}
\vspace{-0.1in}
\end{figure}

\subsection{Discussion}
The success of the proposed protocol is prominently seen in cases where uncoded machine unlearning exhibits significant degradation in performance as unlearning cost decreases. One possible intuition into why this phenomenon occurs is related to the samples used for training each of the weak learners. Influential samples have been explored in the literature extensively \cite{belsley2005regression}. As we previously discussed, coding is a method of combining samples into the coded dataset, including these influential samples.

Let us examine the influence of individual samples on the performance of the trained model. We take two of the previously considered datasets, the Computer Activity dataset \cite{Compact} and the Combined Cycle Power Plant dataset \cite{Dua2019}. We conduct the following experiment: we randomly shuffle the data and split it into training and testing datasets with the same sizes as we used before, then we remove samples from the training dataset according to some criteria, then train a single learner on the remaining samples and observe its test MSE. We use two criteria of removal, the first is as follows: remove a sample if any of its original features lie outside certain percentiles. This criterion removes what we denote as outliers. The other criterion is as follows: remove samples whose original features lie inside certain percentiles, this removes what we denote as inliers. In other words, outliers are samples at the tails of the probability distribution function (PDF), and inliers are the ones close to the median. We vary these percentiles symmetrically on both ends and observe the performance on the testing data for multiple runs then compute the observed average test MSE. Figure \ref{fig:CompAct_inf} shows the experiment results for the dataset with Computer Activity dataset and Figure \ref{fig:CCPP_inf} shows the experiment results for the Combined Cycle Power Plant dataset. In Figure \ref{fig:CompAct_inf}, we see a degradation in performance as more outlier samples are removed which is much more significant and immediate compared to the case where inlier samples are removed. On the other hand, in Figure \ref{fig:CCPP_inf}, the performance of the model after removing outliers and inliers is quite similar until we remove more than $50\%$ of the samples, then a small gap appears between the two curves that gets larger as the number of removed samples increases.

\begin{figure}[t]
\centering
\begin{minipage}{.45\textwidth}
  \centering
  \includegraphics[width=0.95\linewidth]{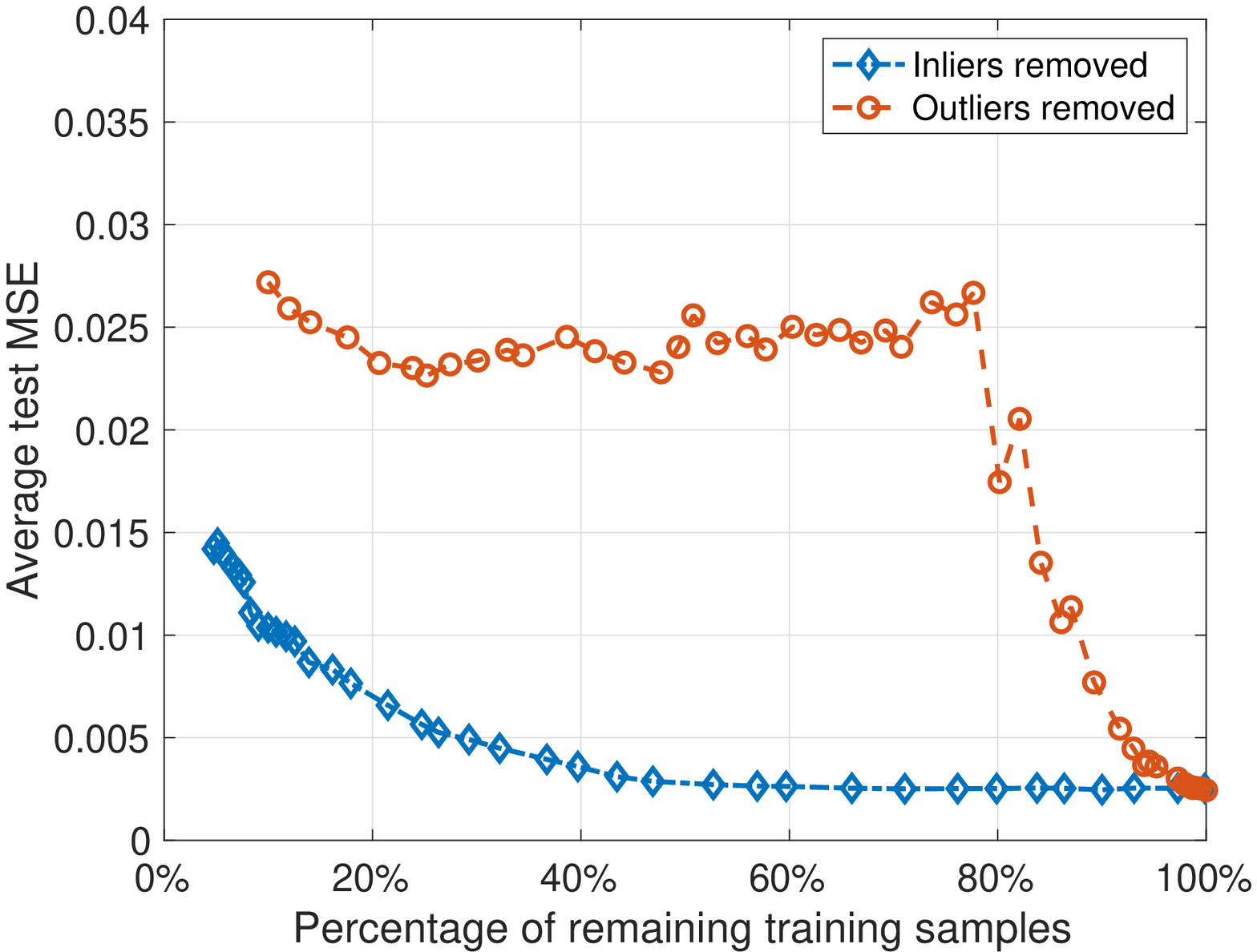}
  \captionof{figure}{Original learning algorithm's performance vs percentage of remaining samples after removal of outliers and inliers from the Computer Activity dataset \cite{Compact}.}
  \label{fig:CompAct_inf}
\end{minipage}
\hfill
\begin{minipage}{.45\textwidth}
    \centering
    \includegraphics[width=0.95\textwidth]{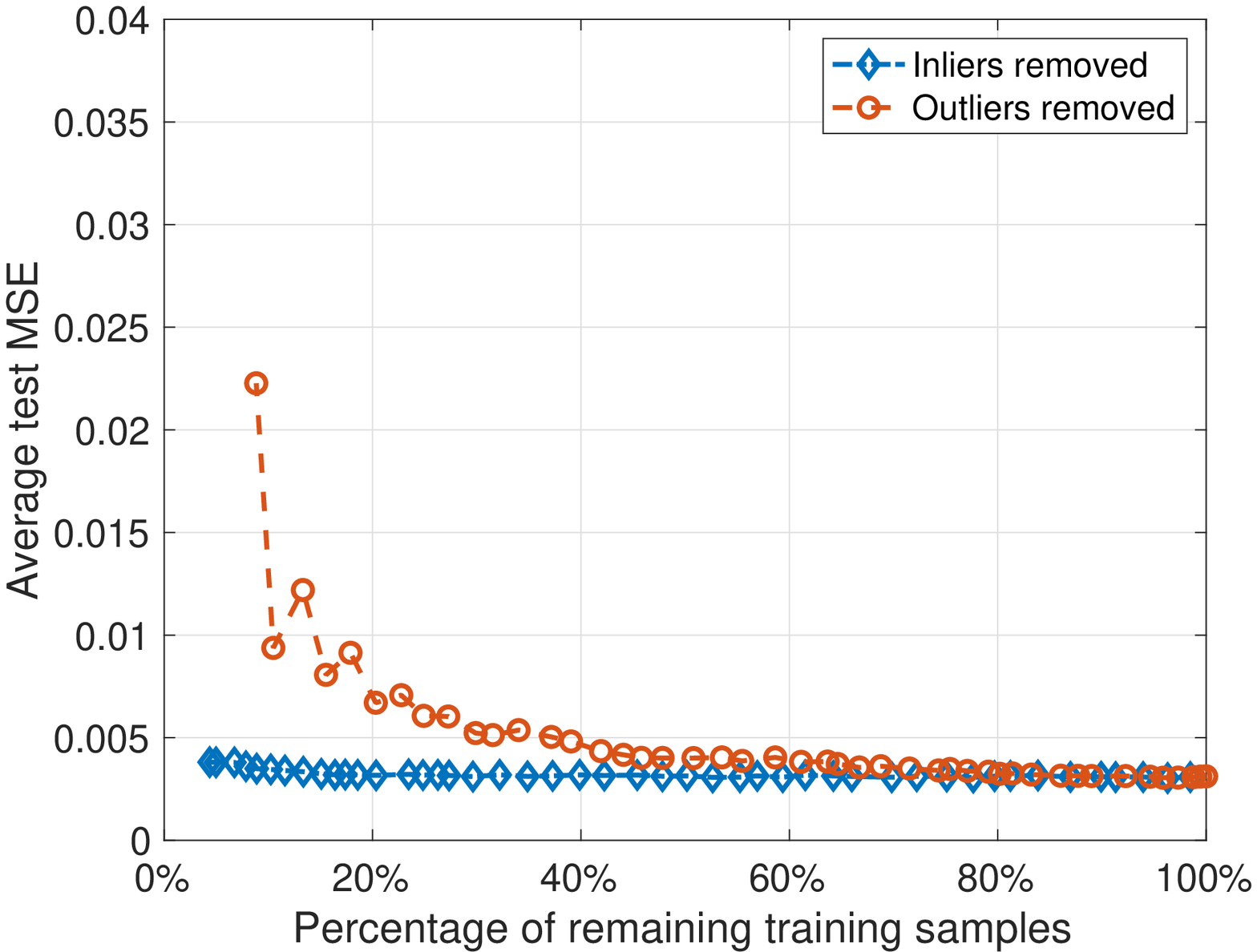}
    \captionof{figure}{Original learning algorithm's performance vs percentage of remaining samples after removal of outliers and inliers from the Combined Cycle Power Plant dataset \cite{Dua2019}.}
    \label{fig:CCPP_inf}
\end{minipage}
\vspace{-0.1in}
\end{figure}

We believe that one explanation behind the behavior of the uncoded machine unlearning is related to the existence of these influential samples, i.e., outlier samples. In particular, if influential samples exist in the dataset, then the uncoded machine unlearning suffers significant degradation as we increase the number of shards and the proposed protocol can provide a better trade-off. On the other hand, if such influential samples do not exist, then the uncoded machine unlearning does not exhibit any degradation in performance as the number of shards increases, and as shown in the experiment in Figure \ref{fig:CCPP}, the proposed protocol does not improve on the uncoded machine unlearning and does not negatively affect it either. It is worth noting that these influential samples exist in heavy-tailed distributions which are quite common in a range of real-world examples such as technology, social sciences and demographics, medicine, etc, where machine learning is increasingly employed for these applications. In the aforementioned experiments, we observed that if the probability distribution functions of some of the features have heavy tails, then we have a trade-off for the uncoded machine unlearning and coding provides a better trade-off. However, if there are no heavy tails in the probability distribution functions then we do not see this trade-off and, hence, coding does not provide a better trade-off.

To verify this observation, we create three synthetic datasets with known feature distribution and known relationships to the response variable. Each one of the datasets has $d=100$ i.i.d. feature vectors whose elements are drawn from $\mathrm{lognormal}(\mu,\sigma^2)$ distribution to create the feature matrix $\bX$. Then, we map these features $\bX$ to a degree 3 polynomial with no interaction terms, resulting in the following
\begin{align}
    \bX_p = [\bX, \bX^{\circ 2}, \bX^{\circ 3}],
\end{align}
where $\bX^{\circ c}$ is the element-wise $c$-th power of matrix $\bX$. The response variable is generated using \eqref{y_projected1} with i.i.d. elements of $\bw$ and $\boldsymbol{\epsilon}$ drawn from the standard normal distribution. The lognormal distribution has two parameters, $\mu$ and $\sigma^2$. We fix $\mu=1$ and vary $\sigma^2$. As $\sigma^2$ increases, the tail becomes heavier; hence, we expect the trade-off to be more evident. The number of samples in each dataset is $25{,}000$ samples, of which $23{,}000$ are used for training and the rest are used for testing. The simulated experiments for $\sigma^2=0.1,0.5,0.7$ are shown in Figure \ref{fig:lognormal_poly}. The code used for all datasets has rate $\tau=5$. As can be seen from the figure, as we increase the value of $\sigma^2$, the tail becomes heavier and the trade-off becomes more significant for the uncoded machine unlearning. Additionally, as the tail becomes heavier, the gain provided by the proposed protocol in terms of the trade-off is more significant compared to the uncoded machine unlearning.

\begin{figure}[t]
\centering
\begin{minipage}{.45\textwidth}
  \centering
  \includegraphics[width=0.95\linewidth]{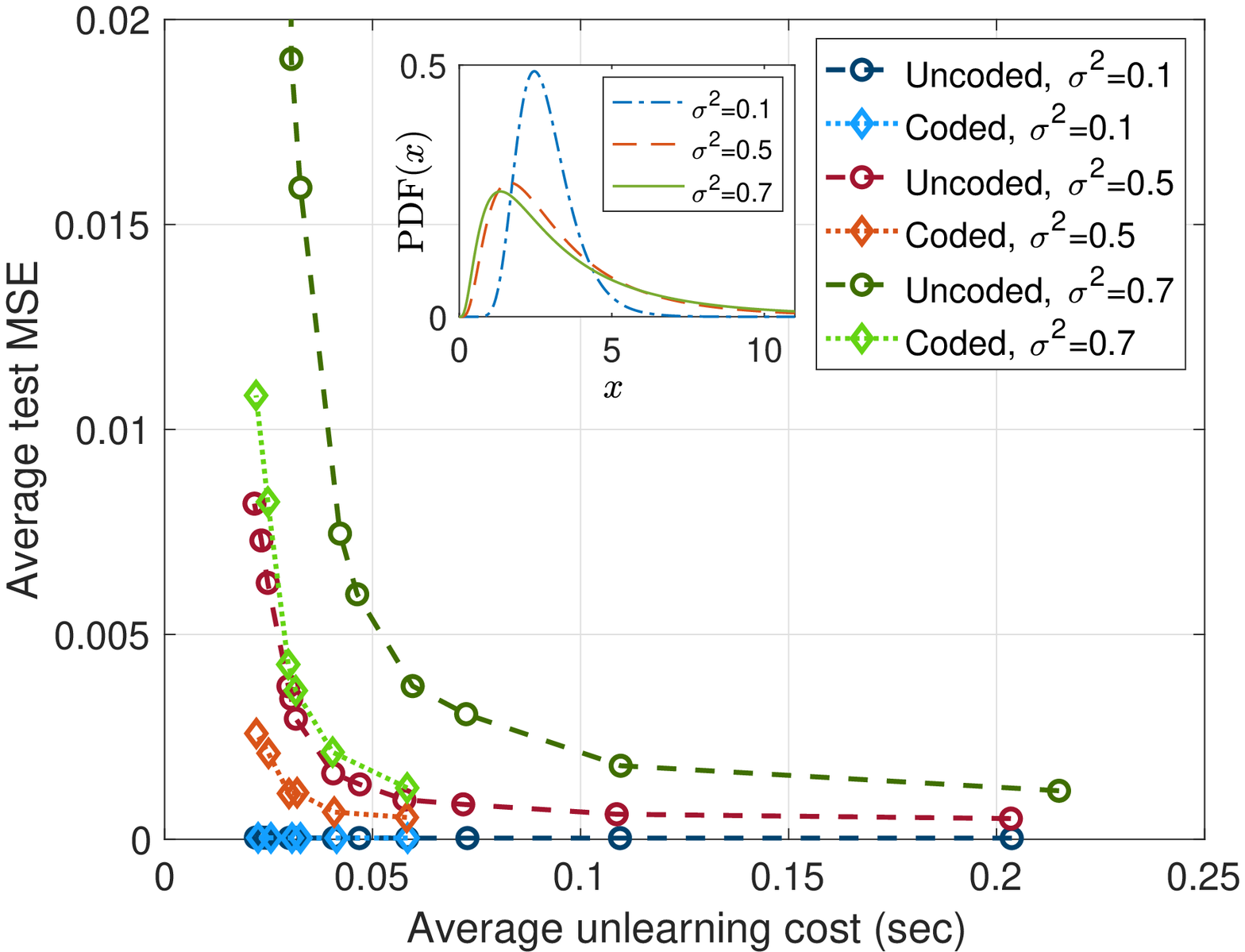}
  \captionof{figure}{Performance vs unlearning cost for synthetic data with lognormal features with fixed $\mu=1$ and different values of $\sigma^2$ used in a polynomial of degree 3. The rate of the code is $\tau =5$. The inset figure shows the PDFs of the original lognormal features of the considered datasets with $\mu=1$ and different values of $\sigma^2$.}
  \label{fig:lognormal_poly}
\end{minipage}
\hfill
\begin{minipage}{.45\textwidth}
    \centering
    \includegraphics[width=0.95\textwidth]{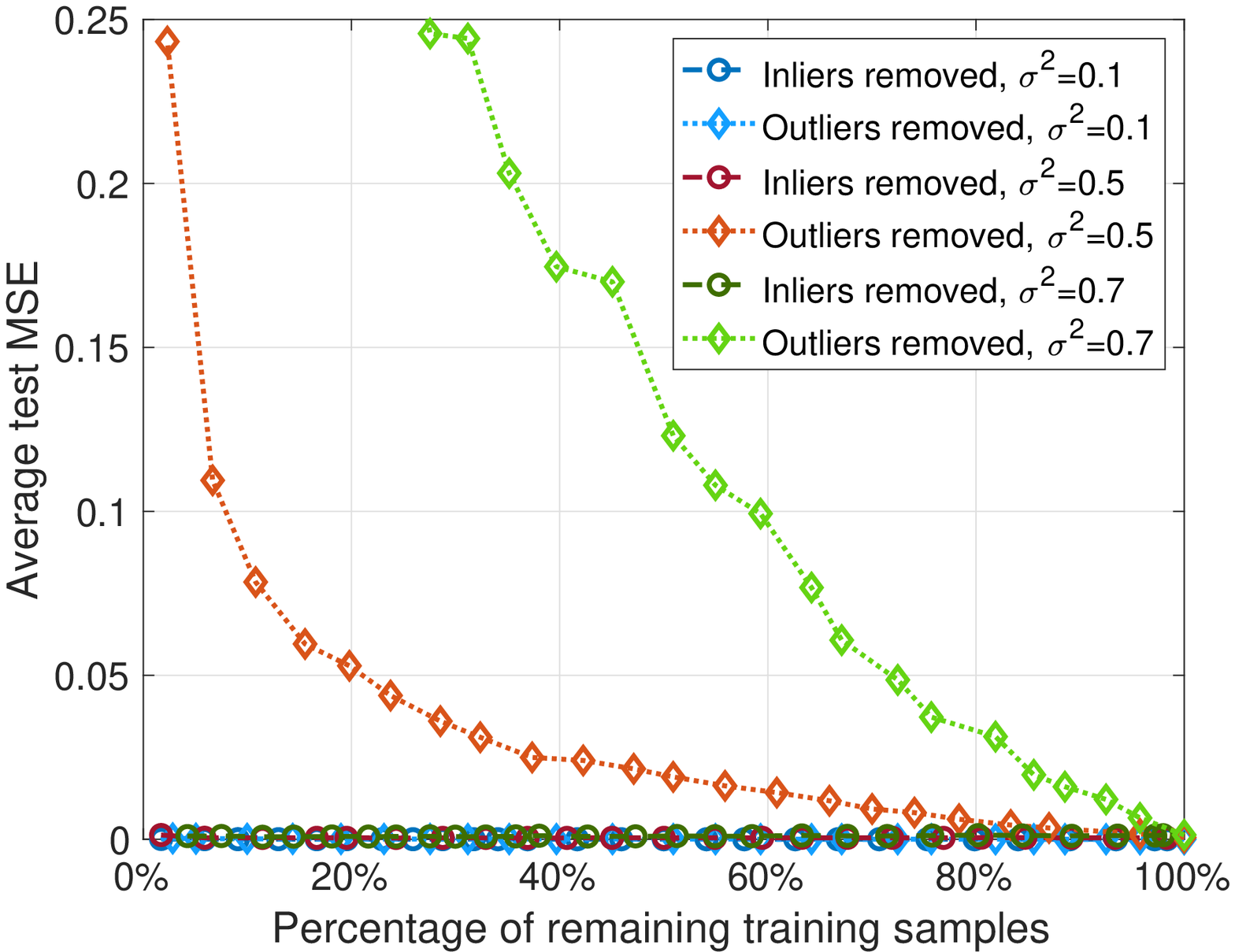}
    \captionof{figure}{Original learning algorithm's performance vs percentage of remaining samples after removal of outliers and inliers from the lognormal polynomial datasets.\\ \vspace{0.32in}}
    \label{fig:lognormal_inf}
\end{minipage}
\vspace{-0.1in}
\end{figure}

We also run experiments, same as the ones for the datasets in Section\,\ref{sec:results} showing the effect of removing inliers versus removing outliers, for these synthetic datasets. We run the inlier and outlier removal process based on the original features and observe the performance of the trained model's performance on the testing dataset. Figure \ref{fig:lognormal_inf} shows results of these experiments. Similar to what we observed in the realistic dataset experiments, for distributions with heavier tails, removing outlier samples has more influence than inlier samples.

\sloppy Additional experiments on synthetic datasets are shown in Appendix \ref{app:synthetic}. They include experiments with known relationships between the features and response variable as well as a dataset generated using features passed through a random 3-hidden-layers MLP to produce the response variable where we utilize random projections to model this relationship.

\section{Conclusion}\label{conclusion}
In this work, we considered the problem of perfect machine unlearning for ensemble learning scenarios where the model consists of a master node and multiple non-communicating weak learners trained on disjoint shards of the training dataset. We focused on the trade-off between the performance and unlearning cost for regression models. We presented a new method of learning called coded learning which can potentially enable more efficient unlearning while exhibiting a better trade-off, in terms of the performance versus the unlearning cost, compared to the uncoded machine unlearning. We presented a protocol for coded learning along with a linear encoder for regression datasets as well as its corresponding unlearning protocol and showed its success in ensuring perfect unlearning. 

We presented a handful of experiments to show the proposed protocol can succeed in providing a better trade-off for various realistic datasets with different values of the underlying parameters. On the other hand, we considered datasets for which the uncoded machine unlearning does not exhibit any trade-off between performance and unlearning cost and showed that coding in these scenarios maintains performance on par with the uncoded machine unlearning. In the experiments, we showed that when using appropriate codes one can potentially reduce the unlearning cost to a fraction of the unlearning cost for a single learner trained on the entire dataset while observing a comparable performance to the one of a single learner. Finally, discussions are provided on whether we should expect the proposed protocol to outperform the uncoded machine unlearning based on the existence of influential samples in the dataset using properties of the probability distribution function of the dataset features.

We consider this work as a first step towards understanding the role of coding in machine unlearning. A few possible directions of future work include extending the proposed protocol to concept classes with higher capacity, such as deep neural networks. Studying different classes of codes beyond linear random codes for supervised learning problems is also another possible avenue for research. Designing protocols for \emph{almost perfect} machine unlearning in convex/non-convex models where only statistical guarantees are required is another area for future investigation. Finally, theoretical exploration of the interplay between influential samples in conjunction with random coding and their impact on the final learned model is another interesting direction for future work.

\bibliographystyle{IEEEtran}
\bibliography{IEEEabrv}

\newpage
\appendix
\section{Synthetic data}\label{app:synthetic}
In this appendix, we further experiment with three synthetic datasets to show the trade-off using the uncoded machine unlearning and the proposed coded machine unlearning. In the first experiment, we randomly generate $d=100$ i.i.d. feature vectors where each element is drawn from $\chi^2(1)$, i.e., a chi-square distribution with 1 degree of freedom. Then, we map these features $\bX$ to a degree 4 polynomial with no interaction terms, resulting in the following
\begin{align}
    \bX_p = [\bX, \bX^{\circ 2}, \bX^{\circ 3}, \bX^{\circ 4}].
\end{align}
Then, the response variable is generated as described in \eqref{y_projected1}, where elements of $\bw$ and $\boldsymbol{\epsilon}$ are i.i.d. and generated randomly from the standard normal distribution. The size of the dataset is $47{,}000$ samples, of which $42{,}000$ samples are for training the rest are for testing. The simulation is run with no regularization term and using codes of rates $\tau=2, 5, 10$. The result of this experiment is shown in Figure \ref{fig:poly_chi}. It can be seen that coding provides better performance compared to the uncoded machine unlearning at a lower unlearning cost. Additionally, it can be observed that different rates allow for different achievable MSE values. As the rate increases, the lowest achievable MSE increases. Hence, similar to what is observed in Figure \ref{fig:LargeScaleNN}, although higher rates reduce the unlearning cost, they might be incapable of achieving some desired performance measures. For example, see the rightmost points in each of the curves for rates $\tau=2, 5, 10$ in Figure \ref{fig:poly_chi}.

In the second experiment, we use a random MLP to create a nonlinear mapping and utilize random projections to simulate the experiment. Specifically, we randomly generate $d=50$ i.i.d. feature vectors whose elements are drawn from a $\mathrm{lognormal}(1,4)$ distribution, then pass these features through a 3-hidden-layers MLP with $50, 25, 50$ nodes for each layer, respectively, each with a sigmoid activation function, followed by a linear output layer with a single node. All the weights and biases of these layers are i.i.d. and are generated from a standard normal distribution. A standard normally-distributed error term is added to the output of the MLP to generate the final response variable. In this experiment, we use random projections as described in \eqref{randproj} on the normalized original features using $D=1{,}000$ and the aforementioned parameters. The size of the dataset is $90{,}000$ samples, of which $82{,}000$ are for training and the rest are for testing. The results of this experiment for a code with rate $\tau = 5$ and regularization parameters $\lambda = 10^{-2}, 10^{-3}$ are shown in Figure \ref{fig:NN_lognormal}. To illustrate the benefit of random projections, compare the curves in the figure with the performance of using the original features and ridge regression on a single learner which is trained on the entire uncoded training dataset where we observe an average MSE in the range $0.147-0.15$ for the aforementioned values of $\lambda$, as well as $\lambda = 0$. This experiment shows that coding can provide gain in the trade-off, compared to the uncoded machine unlearning, even for models that employ regularization using different values of $\lambda$.

\begin{figure}[t]
\centering
\begin{minipage}{.45\textwidth}
  \centering
  \includegraphics[width=0.95\linewidth]{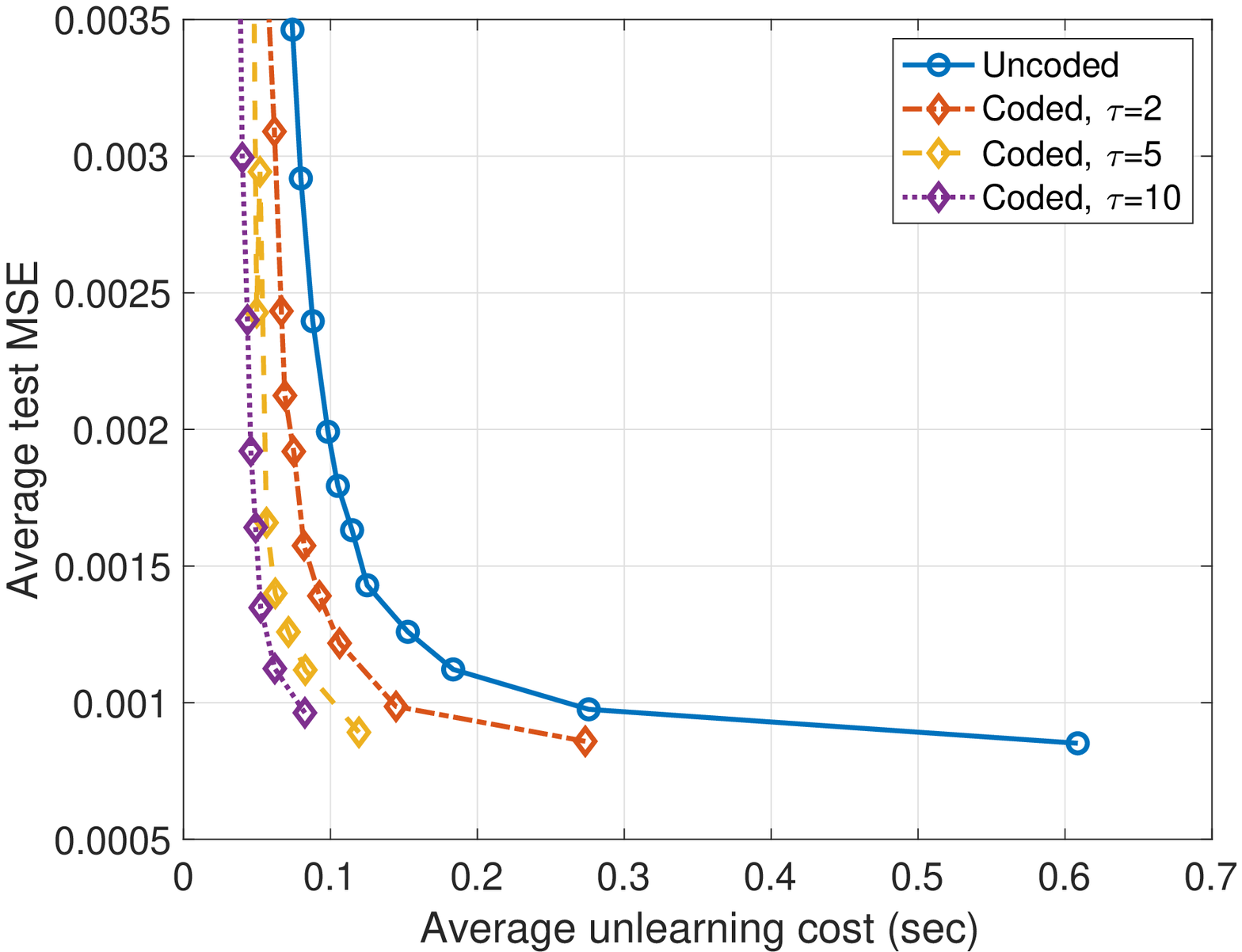}
  \captionof{figure}{Performance vs unlearning cost for synthetic data with $\chi^2(1)$ features used in a polynomial of degree 4.\\ \vspace{0.2in}}
  \label{fig:poly_chi}
\end{minipage}
\hfill
\begin{minipage}{.45\textwidth}
    \centering
    \includegraphics[width=0.95\textwidth]{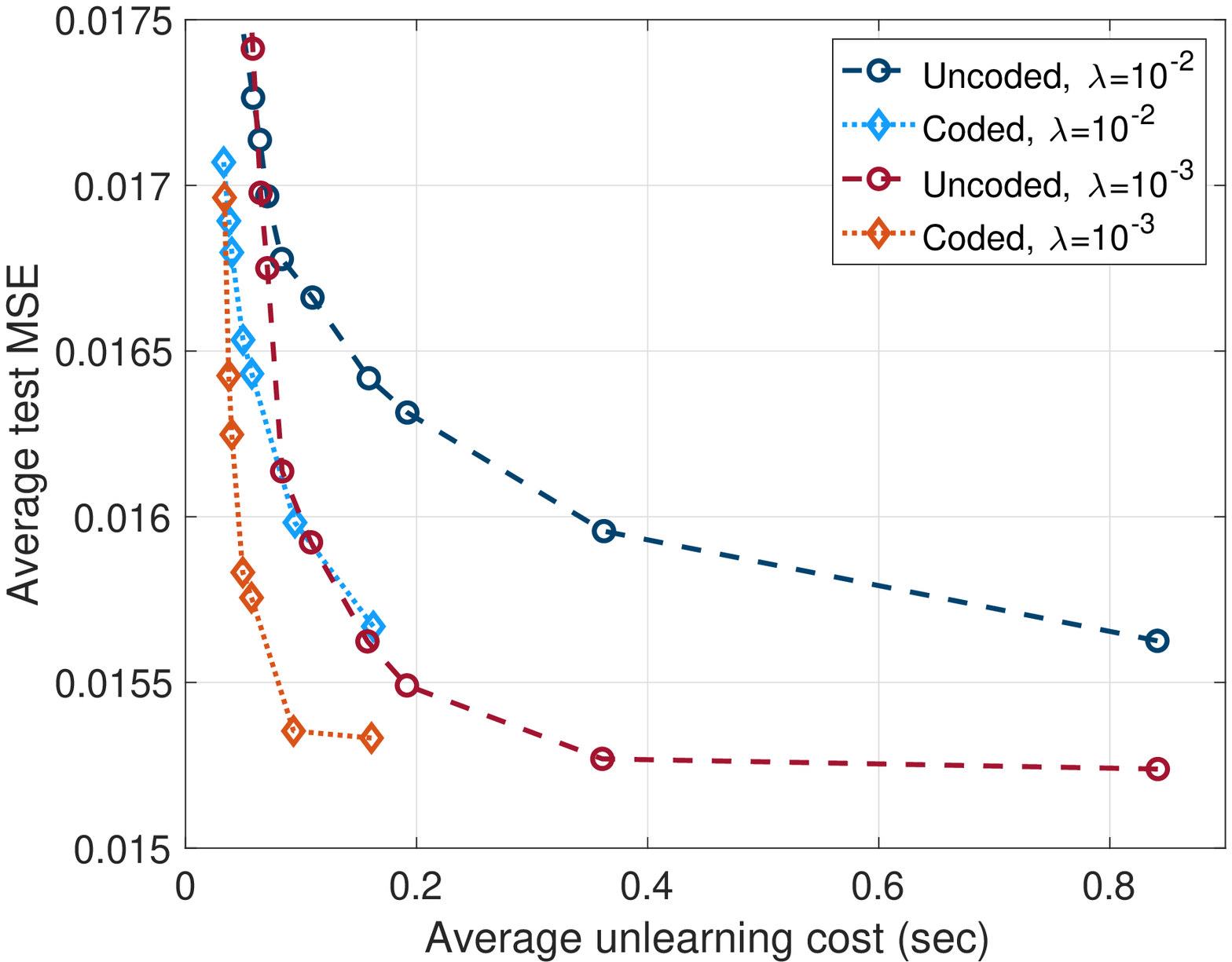}
    \captionof{figure}{Performance vs unlearning cost for synthetic data generated from an MLP with $\mathrm{lognormal}(1,4)$ features using random projections to a $1{,}000$-dimensional, code rate $\tau = 5$, and different values of $\lambda$.}
    \label{fig:NN_lognormal}
\end{minipage}
\vspace{-0.1in}
\end{figure}

In the third experiment, we generate $d=100$ i.i.d. feature vectors where each element of these vectors is drawn from a standard normal distribution. Then, the response variable is generated as a linear combination of these features with additive noise, where elements of $\bw$ and $\boldsymbol{\epsilon}$ are i.i.d. and generated randomly from the standard normal distribution. The size of the dataset is $15{,}000$ samples, of which $10{,}000$ are for training the rest are for testing. We simulate the linear regression problem using codes of rates $\tau=2, 5$. The result of this experiment is shown in Figure \ref{fig:Gaussian}. This experiment shows a case where the uncoded machine unlearning maintains the same performance as the unlearning cost decreases. For this case, there is no meaningful region for the coding to operate in and, intuitively, we do not expect it to beat the performance of the original learning algorithm with a single uncoded shard. Therefore, although coding does not provide a better trade-off in this case, it does not exhibit a worse trade-off either.

\begin{figure}[t]
  \begin{center}
    \includegraphics[width=0.4275\columnwidth]{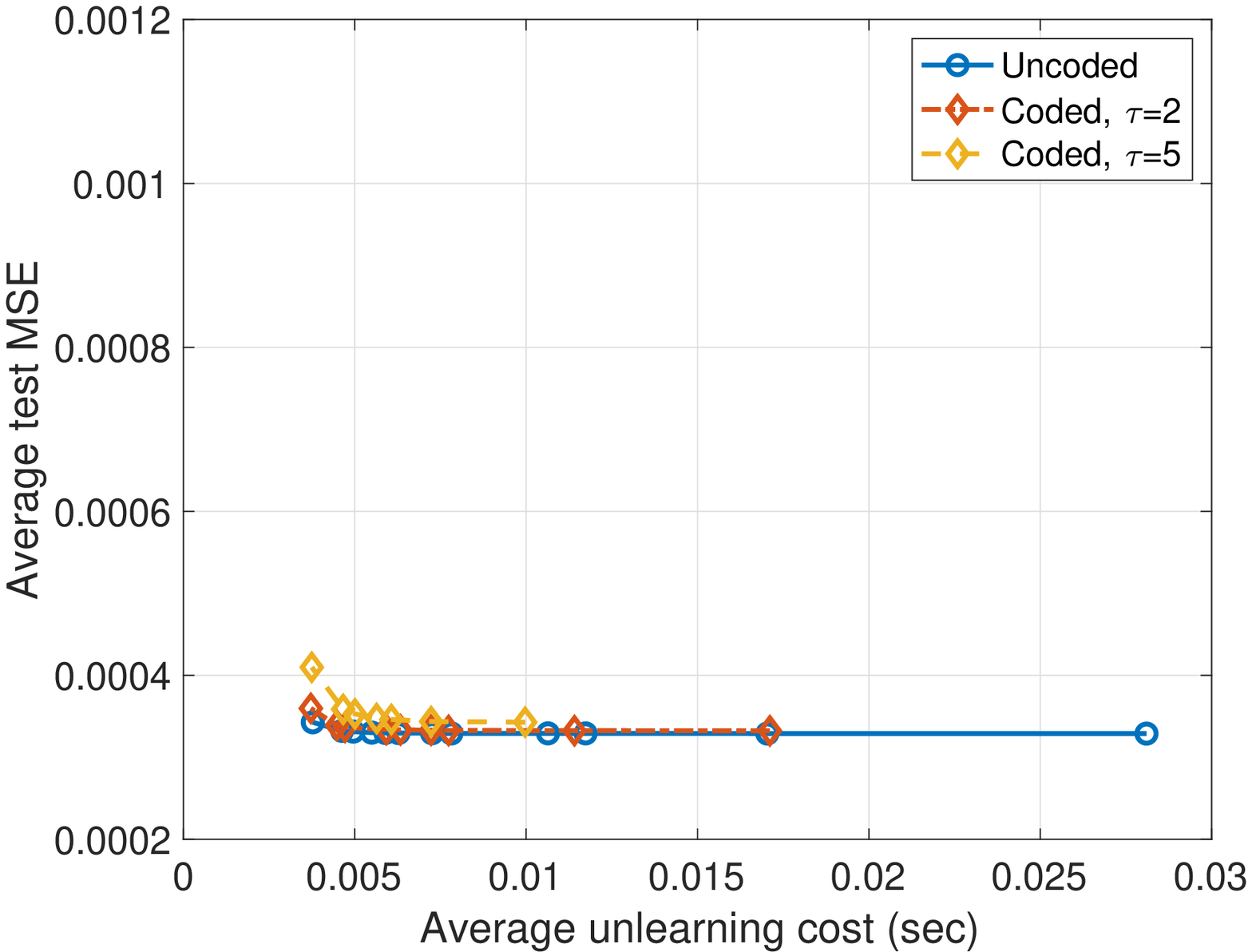}
    \caption{Performance vs unlearning cost for synthetic data with standard normally distributed features used in a linear model.}
    \label{fig:Gaussian}
  \end{center}
\end{figure}

Finally, we conduct two experiments that are concerned with the outlier vs inlier removal for two datasets; the aforementioned chi-square features in a polynomial and the standard normal features in a linear model. The results of the two experiments are shown in Figure \ref{fig:Chi_inf} and Figure \ref{fig:Gaussian_inf}, respectively. Similar to the observations discussed in Section\,\ref{experiments}, the degradation in performance is more evident in datasets whose features have heavier tails. Furthermore, as we observed in the experiments in this appendix, coding provides a better trade-off compared to the uncoded machine unlearning for datasets with heavy-tailed features; however, if the dataset does not have heavy-tailed features, coding does not negatively affect the trade-off.

\begin{figure}[t]
\centering
\begin{minipage}{.45\textwidth}
  \centering
  \includegraphics[width=0.95\linewidth]{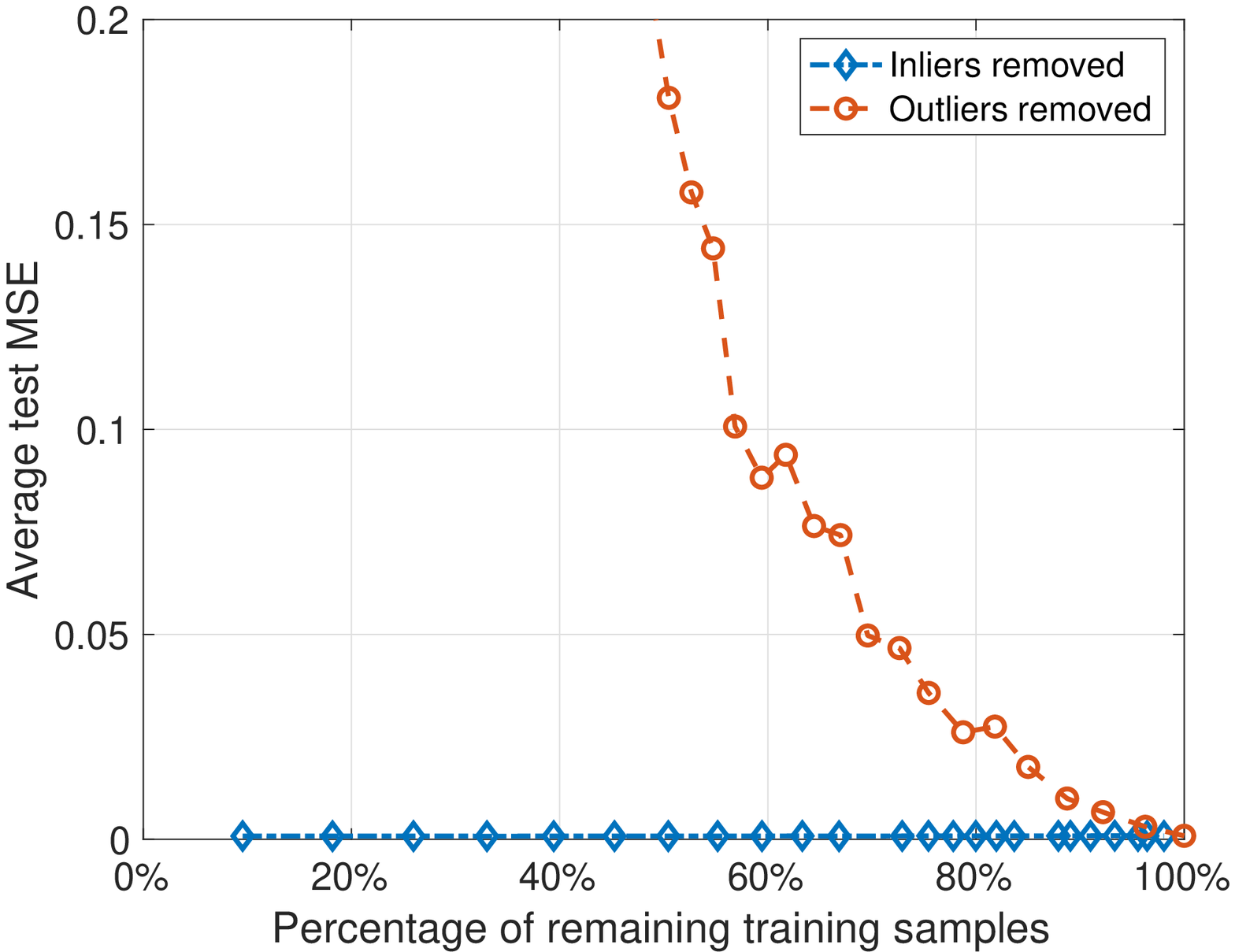}
  \captionof{figure}{Original learning algorithm's performance vs percentage of remaining samples after removal of outliers and inliers from the $\chi^2(1)$ polynomial dataset.}
    \label{fig:Chi_inf}
\end{minipage}
\hfill
\begin{minipage}{.45\textwidth}
    \centering
    \includegraphics[width=0.95\textwidth]{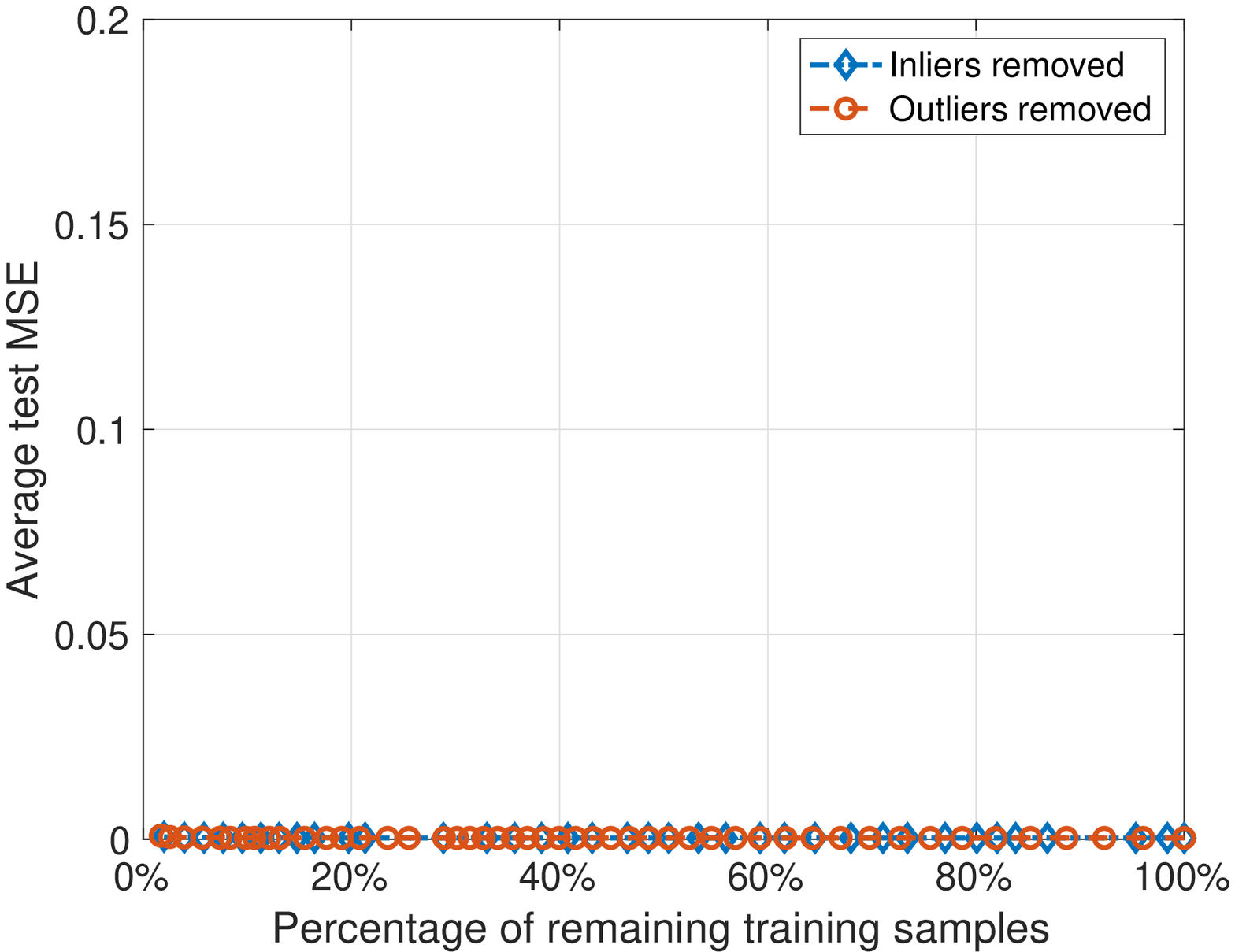}
    \captionof{figure}{Original learning algorithm's performance vs percentage of remaining samples after removal of outliers and inliers from the linear dataset with standard normal features.}
    \label{fig:Gaussian_inf}
\end{minipage}
\vspace{-0.1in}
\end{figure}

\section{Training Performance}
In this appendix, the performance of a model on the training dataset is measured as the average MSE of the uncoded training dataset, regardless of whether coding is utilized or not, using the aggregate model. Note that the learning cost is computed as the average time that is taken to train the weak learners on their respective datasets. The resulting trade-off for the realistic datasets considered in Section \ref{experiments} is shown in Figures \ref{fig:CASP_train}-\ref{fig:CCPP_train}, and for the experiments considered in Appendix \ref{app:synthetic} is shown in Figures \ref{fig:poly_chi_train}-\ref{fig:Gaussian_train}. It can be observed from these figures that the average train MSE is always less than the average test MSE when comparing each point in these figures with its corresponding point in the average test MSE figures. The correspondence here is not an $x$-axis correspondence, but rather an order correspondence. For instance, the rightmost point in any curve should be compared with the rightmost point in the corresponding curve and so on.

\begin{figure}[t]
\centering
\begin{minipage}{.45\textwidth}
  \centering
  \includegraphics[width=0.95\linewidth]{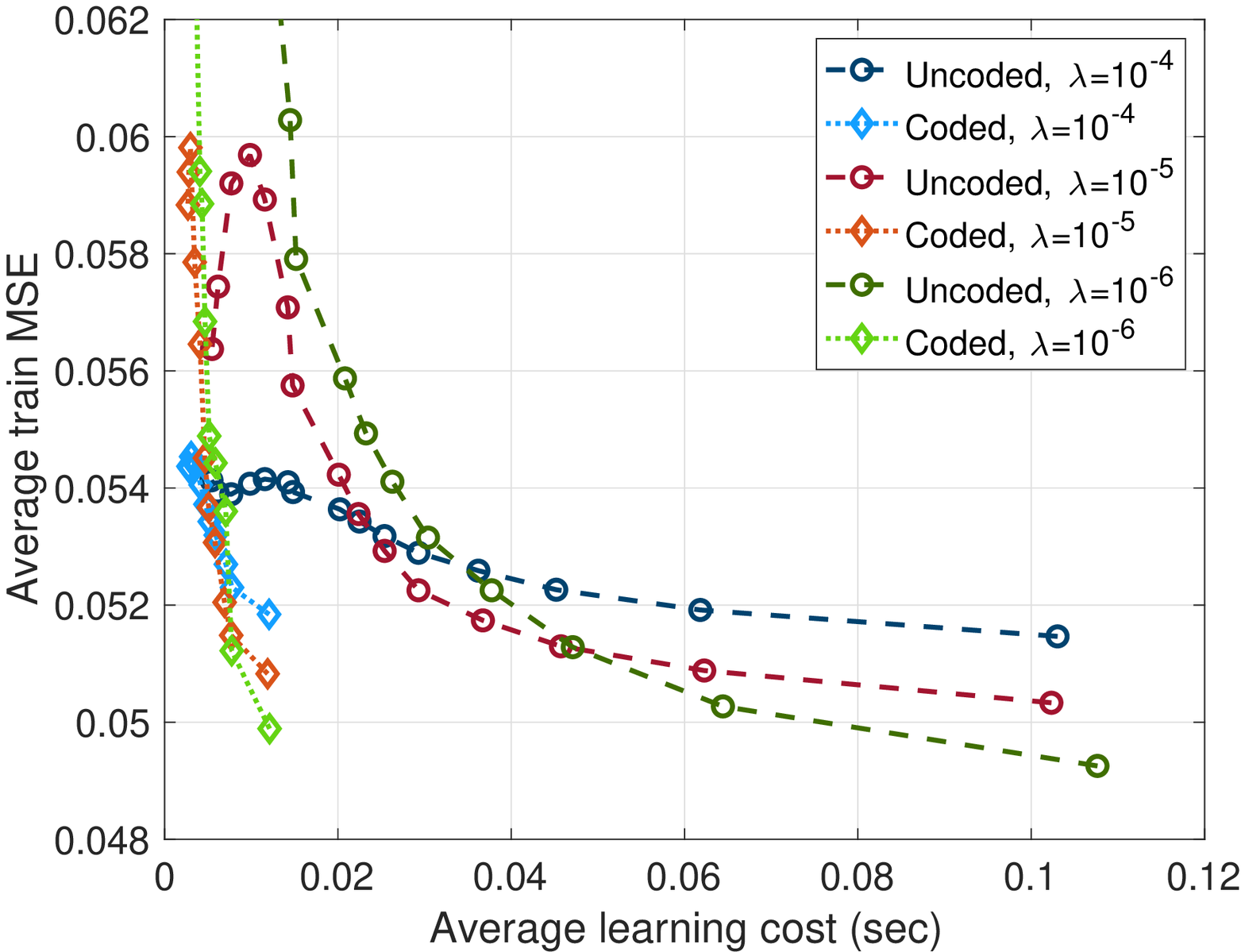}
  \captionof{figure}{Training performance vs learning cost for different values of $\lambda$ using the Physicochemical Properties of Protein Tertiary Structure dataset \cite{Dua2019}, random projections of features to a $300$-dimensional space, and a code of rate $\tau=5$.}
  \label{fig:CASP_train}
\end{minipage}
\hfill
\begin{minipage}{.45\textwidth}
    \centering
    \includegraphics[width=0.95\textwidth]{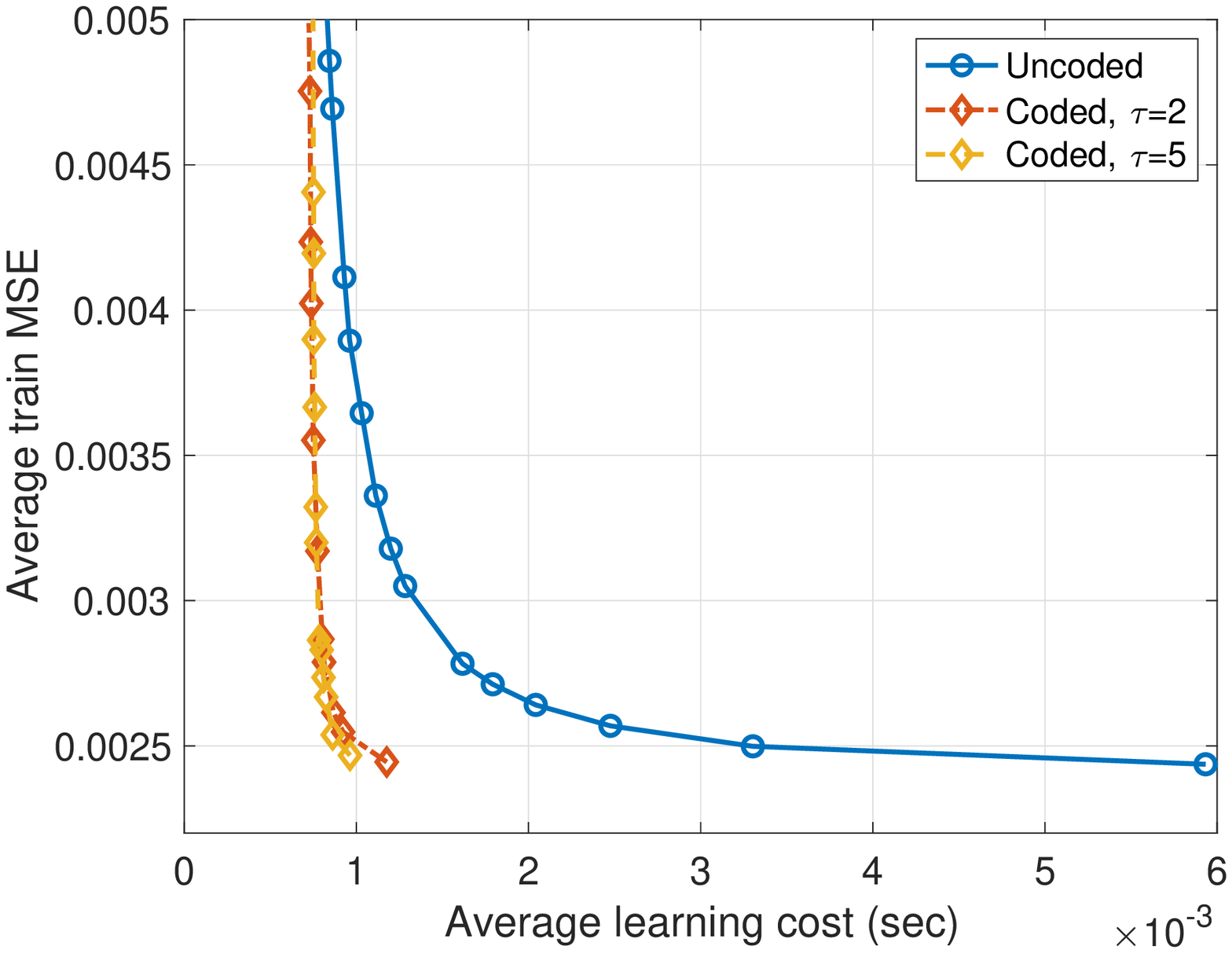}
    \captionof{figure}{Training performance vs learning cost for different rates using the Computer Activity dataset \cite{Compact}, random projections of features to a $25$-dimensional space, and $\lambda=10^{-3}$.\\ \vspace{0.15in}}
    \label{fig:CompAct_train}
\end{minipage}
\vspace{-0.1in}
\end{figure}

\begin{figure}[t]
\centering
\begin{minipage}{.45\textwidth}
  \centering
  \includegraphics[width=0.95\linewidth]{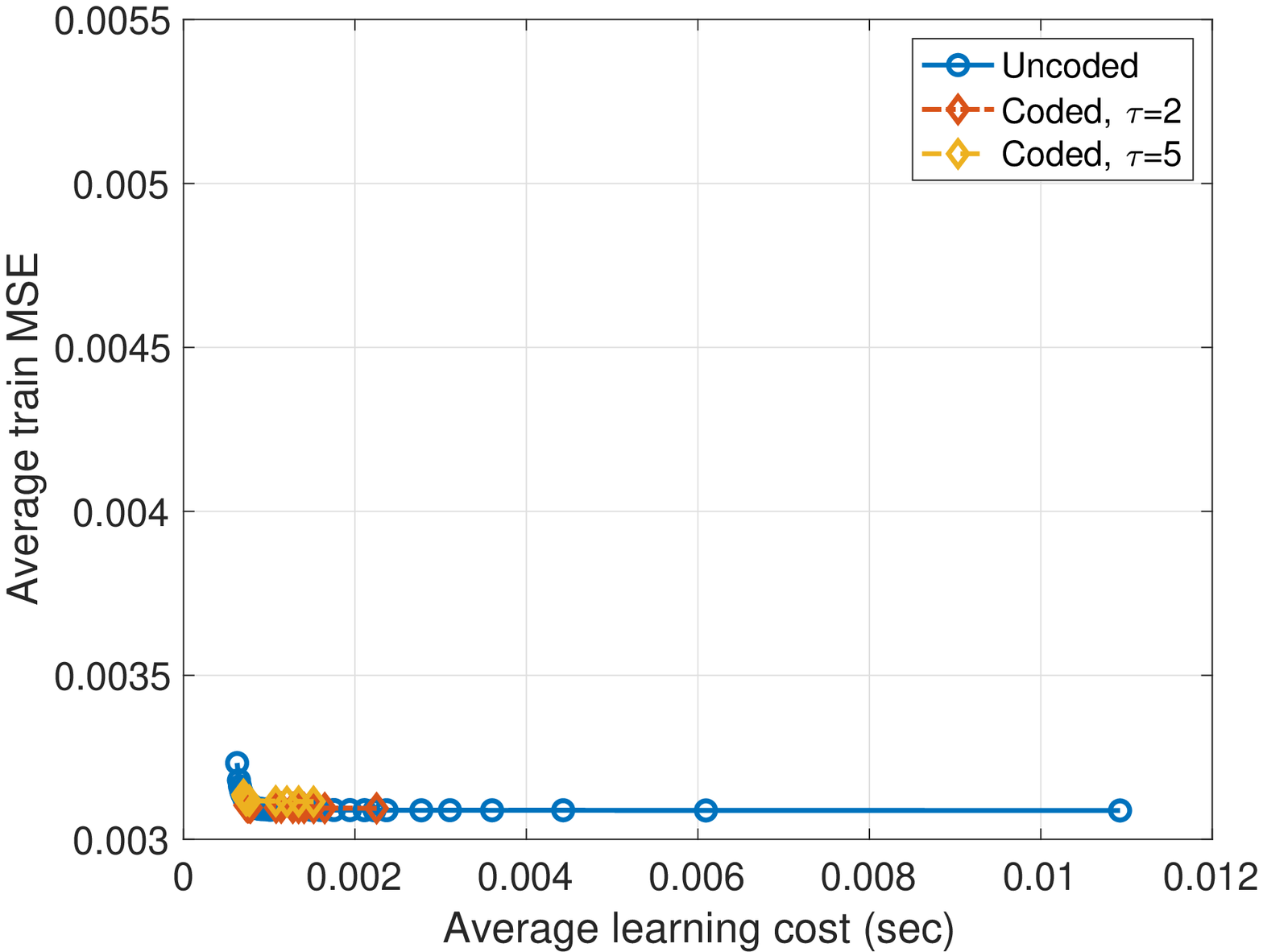}
  \captionof{figure}{Training performance vs learning cost for different rates using the Combined Cycle Power Plant dataset \cite{Dua2019} and random projections of features to a $20$-dimensional space.}
    \label{fig:CCPP_train}
\end{minipage}
\hfill
\begin{minipage}{.45\textwidth}
    \centering
    \includegraphics[width=0.95\textwidth]{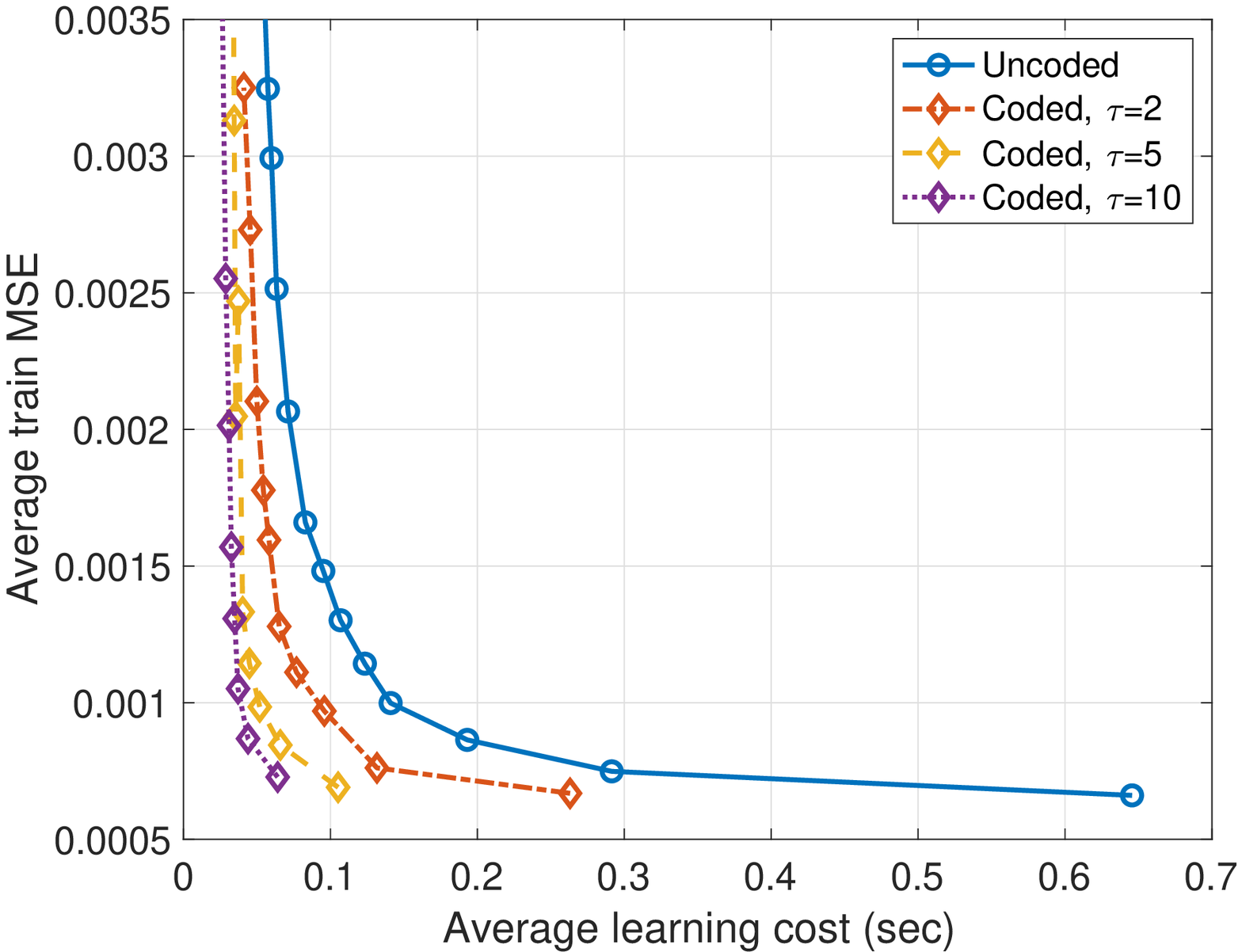}
    \captionof{figure}{Training performance vs learning cost for synthetic data with $\chi^2(1)$ features used in a polynomial of degree 4.\\}
    \label{fig:poly_chi_train}
\end{minipage}
\vspace{-0.1in}
\end{figure}

\begin{figure}[t]
\centering
\begin{minipage}{.45\textwidth}
  \centering
  \includegraphics[width=0.95\linewidth]{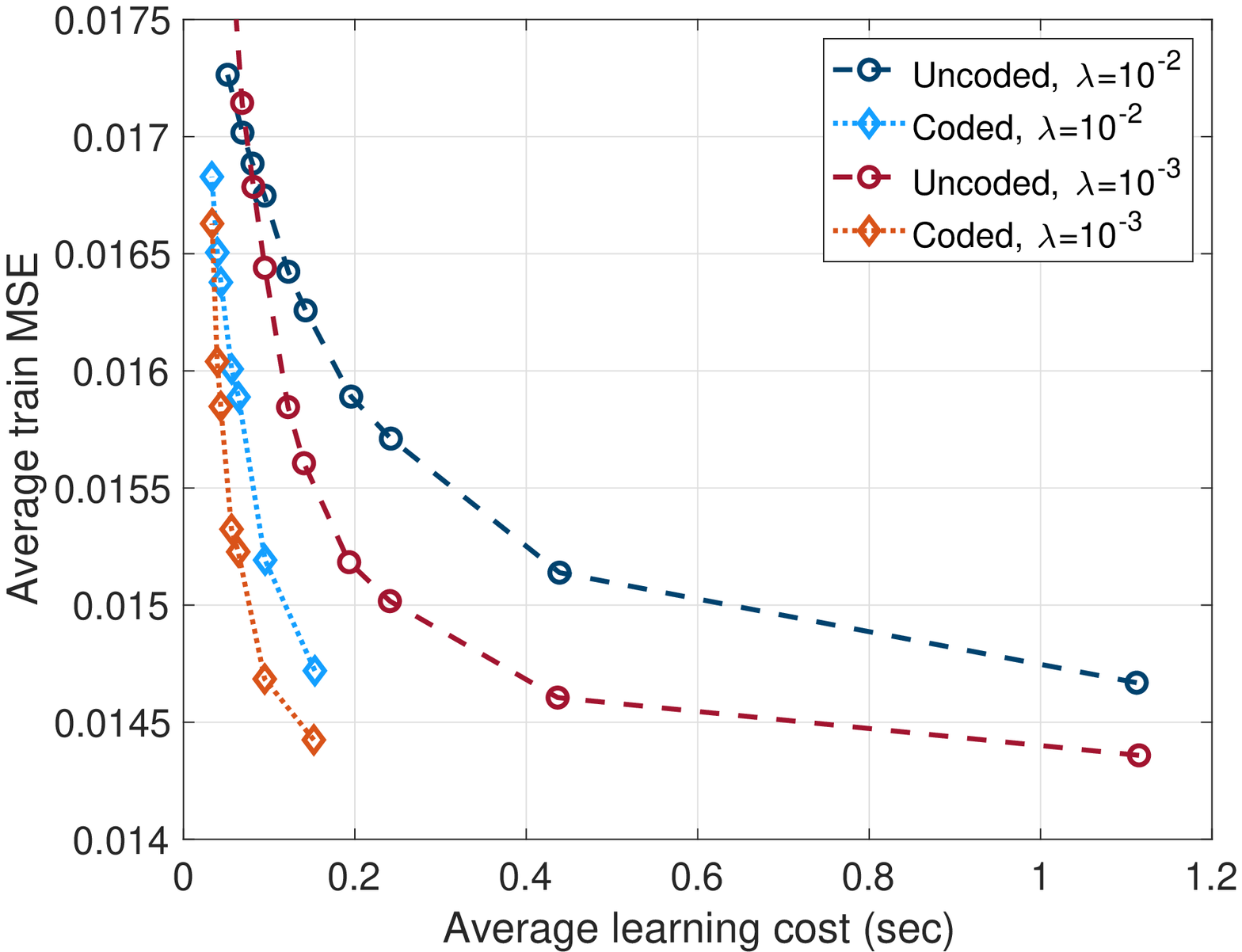}
  \captionof{figure}{Training performance vs learning cost for synthetic data generated from an MLP with $\mathrm{lognormal}(1,4)$ features using a code of rate $\tau = 5$ and different values of $\lambda$.}
    \label{fig:NN_lognormal_train}
\end{minipage}
\hfill
\begin{minipage}{.45\textwidth}
    \centering
    \includegraphics[width=0.95\textwidth]{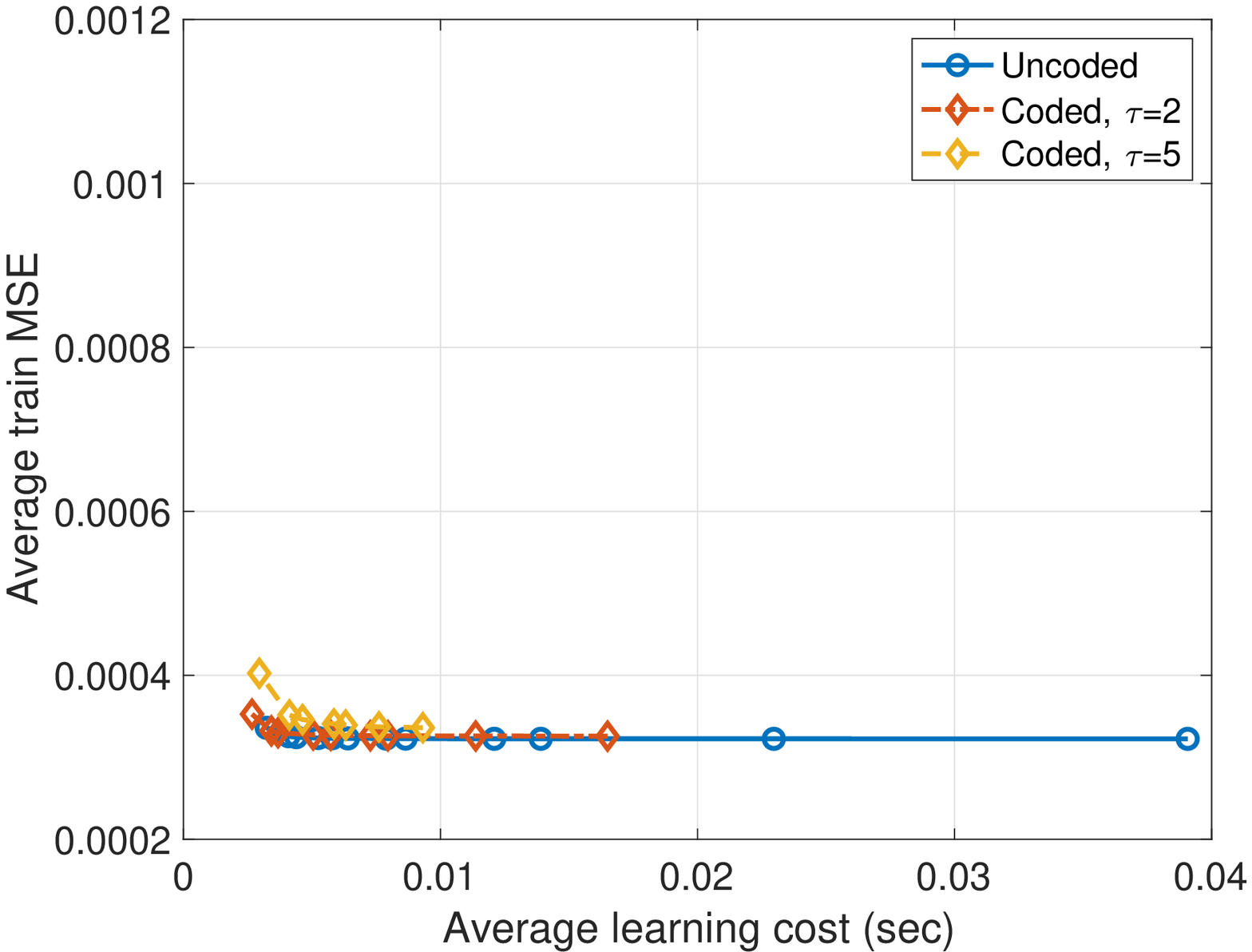}
    \captionof{figure}{Training performance vs learning cost for synthetic data with standard normally distributed features used in a linear model.}
    \label{fig:Gaussian_train}
\end{minipage}
\vspace{-0.1in}
\end{figure}

\end{document}